%% file: main_paper.tex
\documentclass{article}

    \PassOptionsToPackage{numbers, compress}{natbib}

    \usepackage[final]{neurips_2024}

\usepackage[utf8]{inputenc} %
\usepackage[T1]{fontenc}    %
\usepackage[pagebackref,hidelinks]{hyperref}
\usepackage{url}            %
\usepackage{booktabs}       %
\usepackage{amsfonts}       %
\usepackage{nicefrac}       %
\usepackage{microtype}      %
\usepackage{xcolor}         %

\usepackage{multirow}

\usepackage{algorithm}
\usepackage{algorithmic}
\usepackage{enumitem}
\usepackage{subcaption}
\usepackage{wrapfig}

\usepackage[group-separator={,},group-minimum-digits={3}]{siunitx}

\usepackage{amsmath}
\usepackage{amssymb}
\usepackage{mathtools}
\usepackage{amsthm}

\usepackage[capitalize]{cleveref}
\crefname{section}{Sec.}{Secs.}
\Crefname{section}{Section}{Sections}
\Crefname{table}{Table}{Tables}
\crefname{table}{Tab.}{Tabs.}

\theoremstyle{plain}
\newtheorem{theorem}{Theorem}[section]

\theoremstyle{definition}

\theoremstyle{remark}
\newtheorem{remark}[theorem]{Remark}

\title{{O}n {S}ampling {S}trategies for {S}pectral {M}odel {S}harding}

\author{%
  Denis Korzhenkov\\
  Qualcomm AI Research\thanks{Qualcomm AI Research, Qualcomm Technologies Netherlands B.V (Qualcomm AI Research is an initiative of Qualcomm Technologies, Inc.).} \\
  Amsterdam, The Netherlands \\
  \texttt{dkorzhen@qti.qualcomm.com} \\
  \And
  Christos Louizos \\
  Qualcomm AI Research$^*$ \\
  Amsterdam, The Netherlands \\
  \texttt{clouizos@qti.qualcomm.com} \\
}

\begin{document}

\maketitle

\begin{abstract}
\input{neurips_2024/chapters/abstract}
\end{abstract}

\newcommand{\Iindicator}{\mathbb{I}}
\newcommand{\Isubsample}{\mathcal{I}}
\newcommand{\Eexpect}{\mathbb{E}}
\newcommand{\topic}[2][0.5mm]{\vspace{#1}\noindent\textbf{#2}}

\newcommand\lft{\mathopen{}\left}
\newcommand\rgt{\aftergroup\mathclose\aftergroup{\aftergroup}\right}

\input{neurips_2024/chapters/introduction}
\input{neurips_2024/chapters/related_works}

\input{neurips_2024/chapters/method}
\input{neurips_2024/chapters/experiments}
\input{neurips_2024/chapters/conclusion}

\ifnum\value{page}>9 \errmessage{Number of pages exceeded!!!!}\fi

\bibliography{neurips_2024/egbib}
\bibliographystyle{abbrvnat}

\newpage
\appendix
\input{neurips_2024/appendix/additional_discussion}
\input{neurips_2024/appendix/derivations}

\input{neurips_2024/appendix/wallenius}
\input{neurips_2024/appendix/licenses}
\input{neurips_2024/appendix/impact_statement}
\input{neurips_2024/appendix/algorithms_listing}


\end{document}

%% file: neurips_2024/chapters/abstract.tex
The problem of heterogeneous clients in federated learning has recently drawn a lot of attention.
Spectral model sharding, i.e., partitioning the model parameters into low-rank matrices based on the singular value decomposition, has been one of the proposed solutions for more efficient on-device training in such settings.
In this work, we present two sampling strategies for such sharding, obtained as solutions to specific optimization problems.
The first produces unbiased estimators of the original weights, while the second aims to minimize the squared approximation error.
We discuss how both of these estimators can be incorporated in the federated learning loop and practical considerations that arise during local training. 
Empirically, we demonstrate that both of these methods can lead to improved performance on various commonly used datasets. 

%% file: neurips_2024/chapters/introduction.tex
\section{Introduction}
\label{sec:introduction}

Due to increasing concerns about user data privacy, a federated approach becomes a viable alternative for common centralized methods of machine learning~\citep{9599369}.
Unlike centralized training on powerful servers, the federated setting faces a variety of user devices like personal computers, mobile phones, tablets, etc. which may have different constraints in terms of memory footprint, spare CPU resources, battery levels, and many others~\citep{10.1145/3596907}.
Therefore, in practice, the scenario where each user trains the same model may be impossible.

To tackle this problem, several methods which carefully create smaller sub-models from a larger global model were proposed~\citep{diao_heterofl_2022,horvath_fjord_2022,alam2022fedrolex,yao_fedhm_2022}.
The general recipe for each communication round is composed of three steps.
Initially, the server receives from each client the information about its capabilities and sends a sub-model of the corresponding size to that client.
Afterwards, each of the clients update their respective sub-models for a specific number of steps.
Finally, the server collects the fine-tuned models from each client and aggregates them into a single global model.

In order to create sub-models, the recently proposed method called PriSM suggests to conduct singular value decomposition (SVD) of the weight matrices of affine layers and randomly sample the terms from this decomposition~\citep{niu_federated_2022,niu_overcoming_2023}.
The number of sampled terms is a  hyperparameter which depend on the capacity of the client.
Despite the fact that PriSM has shown promising results on several datasets, its sampling design remains heuristic. 

In this work, we show that by introducing certain first principles, better strategies can be obtained via solutions to specific optimization problems. 
More specifically, our aim is to obtain the inclusion probabilities for the SVD factors in a way that minimizes the expected, under these probabilities, squared approximation error.
In isolation, this requirement leads to the deterministic top-$n$ sampling, as the Eckart–Young–Mirsky theorem implies~\citep{Eckart_Young_1936}.
However, we propose two additional principles which lead to two novel strategies, respectively.
The first is inspired by the commonly used \emph{inverted} implementation of dropout regularization~\citep{JMLR:v15:srivastava14a,CS231n}, and searches for the solution in the family of unbiased estimators.
The second one takes into account all the (participating in a given round) client specific sub-models, and aims to reconstruct the full weight matrix from all these observations (instead of considering each client separately).
Notably, the solutions for both strategies can be presented in closed form. Furthermore, since the sampling of sub-models takes place on the server side, the computation of the optimal probabilities in these settings does not require any additional calculation on potentially weak clients.
Finally, we  employ such procedures for sampling without replacement which preserve given inclusion probabilities~\citep{tille_sampling_2006}.

%% file: neurips_2024/chapters/related_works.tex
\section{Related Works}
\label{sec_related_works}

Heterogeneity in its different aspects is ubiquitous in real-life setups for federated learning and is therefore a long-standing topic of research~\citep{gao2022survey,10.1145/3625558}.
Probably the most well-studied issue is the non-i.i.d. distributed data samples, available on clients' devices~\citep{hsu_nonidentical_2019}.
However, the deployment of federated systems also faces the variety of devices which often have different computational constraints~\citep{10.1145/3596907}.
This makes training the same model for each client infeasible and motivates the search for new solutions.

Many popular approaches to tackle this kind of heterogeneity follow one of two directions.
Methods based on knowledge distillation in most cases use client models to train a new predictor on the server side~\citep{NEURIPS2020_18df51b9,10.5555/3495724.3496904}.
However, these techniques typically require a separate dataset on the server, which may be a serious limitation in practice.
In contrast, partial training-based approaches do all the training solely on the client side.
To fit the clients' constraints, they sample sub-models which have a size that allows  them to be trained on the corresponding device. After local training, these sub-models are aggregated into a global model.
Notably, this allows for a final model size that is otherwise infeasible to train in any of the available clients.

Many such methods conduct the sub-sampling in the original model space, e.g., by selecting specific output dimensions on each layer.
While HeteroFL~\citep{diao_heterofl_2022} simply selected the top-left corner of the weight matrix of the required size, methods such as Federated Dropout~\citep{caldas2019expanding} and FjORD~\citep{horvath_fjord_2022} argued for randomized selection.
FedRolex~\citep{alam2022fedrolex} replaced randomness with rolling windows, with the aim of covering all possible sub-models in a shorter period of time.

Recent advances in using factorized low-rank models, both for training from scratch~\citep{khodak_initialization_2022,kamalakara_exploring_2022,sui2023elrt,wang_cuttlefish_2023} and for fine-tuning~\citep{hu2022lora}, found applications in the realm of federated learning~\citep{babakniya2023slora,zhang-etal-2023-fedpetuning}. 
E.g., FLANC used this idea to address heterogeneous computational constraints by representing a weight matrix as a product of two factors, one fixed for all clients and the other having a varying size based on the capacity of each client~\citep{mei2022resourceadaptive}.
Such approaches for model sub-sampling can be generally useful, given the fact that the  (stable) rank of the model weights tends, in some cases, to decrease during training~\citep{pmlr-v201-timor23a,wang_cuttlefish_2023}.
In a similar spirit, FedHM~\citep{yao_fedhm_2022} applied truncated SVD to create the sub-models, while PriSM~\citep{niu_federated_2022,niu_overcoming_2023} employed a random procedure for choosing the SVD factors.
In our work, we propose novel strategies for the selection of sub-models and their training in such scenarios and demonstrate their advantages.

%% file: neurips_2024/chapters/method.tex
\section{Method}
\label{sec:method}

\subsection{Preliminaries}
\label{subsec:preliminaries}

In the case of cross-device federated learning, low-end clients may not be able to perform the computations with the full weight matrix $W$ in all the layers due to various constraints.
A possible remedy for this problem can be to use an approximation for $W$.
In this work we consider a particular case of such approximations, namely sub-sampling the terms from the singular decomposition (SVD) $W = \sum_{i=1}^N \lambda_i u_i v_i^T$, where both sets of column vectors $\left\{ u_i \in \mathbb{R}^{c_{out}} \right\}_i$ and  $\left\{ v_i \in \mathbb{R}^{c_{in}} \right\}_i$ are orthonormal, and the singular values are sorted in a non-increasing order $\lambda_1 \geq \lambda_2 \geq \dots \geq \lambda_N > 0$.
Due to the usage of singular decomposition, we name this type of sub-model sampling as \emph{spectral model sharding}.

More formally, let $0 < r \leq 1$ be the \emph{keep ratio} which depends on the client's capabilities, then the goal is to sample the $n = \left\lceil N r \right\rceil$ terms from the full sum.
Following prior works, keep ratio $r$ is shared across layers~\citep{niu_overcoming_2023,niu_federated_2022,mei2022resourceadaptive}.
With a sufficiently low value of $r$, this  reduces both communication costs and computational complexity.
We introduce a vector of indicators $z = \left(z_1,\dots,z_N\right) \in \left\{0,\,1\right\}^N$ such that $\sum_{i=1}^N z_i = n,$ and the corresponding \emph{sub-layer} is created with a weight matrix $\hat{W},$
\begin{equation}
\label{eq:general_estimator}
\hat{W} = \sum_{i=1}^N z_i \omega_i \lambda_i u_i v_i^T = \sum_{j=1}^n \omega_{\Isubsample_j} \lambda_{\Isubsample_j} u_{\Isubsample_j} v_{\Isubsample_j}^T, 
\end{equation}
where $\left\{\omega_i \in \mathbb{R}\right\}_i$ are some scalar \emph{auxiliary multipliers}, and $\Isubsample$ is the set of selected indices, $\Isubsample = \left\{i:\, z_i=1\right\}.$

If the sampling of indicators $z_i$ is random, we denote the corresponding marginal inclusion probabilities as $\pi_i = \Pr\lft(z_i = 1\rgt).$
The joint distribution of indicator variables is interchangeably referred to as $p\lft(z\rgt)$ or $p\lft(\Isubsample\rgt).$
Wherever the expectation symbol $\Eexpect$ is used in the text, expectation w.r.t. $p\lft(z\rgt)$ is assumed.
If it is necessary to emphasize that the particular estimator, multiplier, vector, etc. is used by the client $c$, this is shown with the client index in parentheses, e.g. $\hat{W}^{\left(c\right)}$.

Based on the previous models of FedHM~\citep{yao_fedhm_2022} and PriSM~\citep{niu_federated_2022,niu_overcoming_2023}, we consider the following pipeline of sharded model training in the case of federated learning on heterogeneous devices:
\begin{enumerate}[leftmargin=*,nosep]
    \item In the beginning of each communication round, the server performs SVD for the weight matrices of the affine layers. Then, for each client $c$ it randomly samples a subset of indices $\Isubsample^{\left(c\right)}$ from the decomposition with keep ratio $r^{\left(c\right)}$, 
    $\hat{W}^{\left(c\right)} = \sum_{j=1}^{n^{\left(c\right)}}  \omega_{\Isubsample_j^{\left(c\right)}} u'_{\Isubsample_j^{\left(c\right)}} v_{\Isubsample_j^{\left(c\right)}}'^{T}$. The singular values are absorbed into the vectors, namely $u'_i = \sqrt{\lambda_i} u_i, \, v'_i = \sqrt{\lambda_i} v_i$.
    Note that in both FedHM and PriSM, the auxiliary multipliers $\omega_i$ (which can be used to `compensate' for the dropped terms) were set to 1.
    However, since they are actively used in our method, we include them in the description.
    \item Each participating client $c$ receives the matrices 
    $U^{\left(c\right)} = \left(u'_{\Isubsample_1^{\left(c\right)}}, \dots,  u'_{\Isubsample_n^{\left(c\right)}} \right) \in \mathbb{R}^{c_{out} \times n}$ 
    and 
    $V^{\left(c\right)} = \left(v'_{\Isubsample_1^{\left(c\right)}}, \dots,  v'_{\Isubsample_n^{\left(c\right)}} \right) \in \mathbb{R}^{c_{in} \times n}$ 
    as well as the vector $\omega^{\left(c\right)} = \left(\omega_{\Isubsample_1^{\left(c\right)}}, \dots,  \omega_{\Isubsample_n^{\left(c\right)}} \right) \in \mathbb{R}^{n}$
    from the server and performs local training with the factorized weight matrix $U^{\left(c\right)} \Omega^{\left(c\right)} V^{\left(c\right)T}$ for some predefined number of epochs. 
    During training $U^{\left(c\right)}$ and $V^{\left(c\right)}$ are updated while the diagonal matrix $\Omega^{\left(c\right)}$ which has values of $\omega^{\left(c\right)}$ on the diagonal is kept frozen.
    When the local training is done, the updated matrices $U^{\left(c\right)}$ and $V^{\left(c\right)}$ are sent back to the server.
    \item The server aggregates each vector from the decomposition separately, i.e.
    \begin{math}
        u'_i \leftarrow \frac{\sum_{c:\,z_i^{\left(c\right)} = 1} D^{\left(c\right)} u_i'^{\left(c\right)}}{\sum_{c:\,z_i^{\left(c\right)} = 1} D^{\left(c\right)}} , 
    \end{math}
    where $D^{\left(c\right)}$ is the size of the local dataset of the client $c$ and $u_i'^{\left(c\right)}$ is the updated value of the $i$-th vector received from the client $c$.
    If the $i$-th term was not selected by any client in the current round, it remains the same.
    Vectors $\left\{v'_i\right\}_i$ are aggregated in the same way.
    Afterward, the updated full weight matrix is calculated as $W \leftarrow \sum_{i=1}^N u'_i v_i'^T$, and a new communication round begins.
\end{enumerate}
FedHM proposed to use the $n$ terms corresponding to the largest singular values, a deterministic selection instead of random sampling.
In contrast, PriSM presented a randomized algorithm, which samples the terms according to the magnitudes of the singular values. It employs the NumPy~\citep{harris2020array} method $\texttt{numpy.random.choice}$ with the option $\texttt{replace=False}$, and the associated probabilities are set proportional to $\lambda_i^k$, where $k$ is a hyperparameter depending on the keep ratio $r$.
Sampling in this manner results in a probability law similar to Wallenius' noncentral hypergeometric distribution~\citep{Wallenius_1963,2e0c8582-b4c5-3cc1-897d-1ca8cebe1269,fog_sampling_2008,fog_calculation_2008}, see~\cref{sec:wallenius} for details. 
Such sampling procedures make the analysis of the resulting marginal inclusion probabilities $\{\pi_i\}_{i}$ complicated~\citep{tille_remarks_2023}.
Thus, the statistical properties of the obtained matrix approximation are opaque.

\subsection{Sampling Distribution}
\label{subsec:inclusion_proba}
As opposed to the prior works, we  consider the search of an optimal sampling distribution to be the central task.
We introduce two assumptions which we use to derive our two different strategies.

\subsubsection{Unbiased Estimator}
\label{sec:unbiased_estimator}
The bias and variance are inherent trade-offs for an estimator. 
To keep the bias in check, we propose an assumption of the unbiasedness of the sampled estimator $\hat{W}$ of the original weight matrix $W$, i.e. $\mathbb{E} \hat{W} = W$, where the expectation is taken over the sampling distribution $p\lft(\Isubsample\rgt)$ with the marginal inclusion probabilities $\left(\pi_1,\dots,\pi_N\right).$
As shown in~\cref{sec:derivation_unbiased_estimator}, this requirement necessarily leads to the following values of auxiliary multipliers $\omega_i = \pi_i^{-1}.$
This selection of multipliers is known in literature as a Horvitz-Thompson estimator~\citep{Horvitz_Thompson_1952} and resembles the common implementation of \emph{inverted dropout}~\citep{JMLR:v15:srivastava14a,CS231n} which assumes dividing the non-zeroed activations by the inclusion probability at training time. %

With the bias of the estimator taken under control, we aim to reduce the approximation error of the resulting layer. 
For any input vector $x \in \mathbb{R}^{c_{in}},$ the following inequality holds 
\begin{equation}
\label{eq:operator_and_frobenius_norm}
    \left\lVert Wx - \hat{W}x \right\rVert_2^2 \leq \left\lVert W - \hat{W} \right\rVert_2^2 \cdot \left\lVert x \right\rVert_2^2 \leq \left\lVert W - \hat{W} \right\rVert_F^2 \cdot \left\lVert x \right\rVert_2^2,
\end{equation}
where $\left\lVert \cdot \right\rVert_F^2$ is a squared Frobenius norm which equals the sum of the squared singular values.
Thus, we can control the upper-bound of the expected error by searching for sampling distribution $p\lft(\Isubsample\rgt)$ that minimizes the \emph{Frobenius discrepancy}
\begin{equation}
\label{eq:general_frob_error}
\resizebox{0.25\textwidth}{!}{$
\min\limits_{p\left(\Isubsample\right)} \quad
\Eexpect_{p\left(\Isubsample\right)} \left\lVert W - \hat{W} \right\rVert_F^2 .
$}
\end{equation}

\begin{theorem}
\label{th:unbiased_estimator}
   For an unbiased estimator $\hat{W}$ of the type specified in~\cref{eq:general_estimator} and consisting of $n$ terms, the Frobenius discrepancy can be expressed in terms of the marginal inclusion probabilities
   \begin{equation}
   \label{eq:unbiased_frob_error}
       \Eexpect \left\lVert W - \hat{W} \right\rVert_F^2  
       = \sum_{i=1}^N \lambda_i^2 \left( \pi_i^{-1} -1 \right),
   \end{equation}
    and the optimal set of inclusion probabilities has the following form
    \begin{equation}
    \label{eq:unbiased_pi_value}
    \pi_i = \begin{cases}
        1, &\textrm{if} \;i \leq t,\\
     \frac{\left(n - t\right)\lambda_i}{\sum_{k=t+1}^N \lambda_k}, &\textrm{if} \;i > t
    \end{cases}
    \end{equation}
    for $i=1,\dots,N,$ where $t \in \left\{0,\dots,n-1\right\}$.
\end{theorem}
\begin{proof}
See~\cref{sec:derivation_unbiased_estimator}.
\end{proof}
As follows from~\cref{th:unbiased_estimator}, to find the true arguments of the minima of~\cref{eq:unbiased_frob_error}, one needs to sweep over $n$ possible values of $t$, evaluate the Frobenius discrepancy and select the $t$ with the minimal discrepancy.
Note that this procedure takes place on the server side and therefore does not require extra computation on clients' devices.

\subsubsection{Collective Estimator}
\label{sec:collective_estimator}
Due to the well-known bias-variance trade-off, unbiased estimators in practice can have too large variance.
This motivates us to consider another perspective.
Consider the simplifying assumption where $C$ clients participating in the current communication round share the same number of terms $n$ in their respective estimators. We can target `reconstructing' the full matrix $W$ from the whole set of i.i.d. observations $\left\{ \hat{W}^{\left(c\right)} \right\}_{1\leq c \leq C}.$
The simplest way of obtaining such reconstruction is averaging, i.e. $\bar{W} = \frac{1}{C} \sum_{c=1}^C \hat{W}^{\left(c\right)},$ and we refer to the matrix $\bar{W}$ as the \emph{collective estimator}.
To ensure that our reconstruction is accurate, we aim to find the optimal sampling distribution as well as the set of auxiliary multipliers,
\begin{equation}
\label{eq:collective_frob_error}
\min\limits_{p\lft(\Isubsample\rgt), \left\{\omega_i\right\}_i} \quad
\Eexpect_{p\lft(\Isubsample\rgt)} \left\lVert W - \bar{W} \right\rVert_F^2 .
\end{equation}

\begin{theorem}
\label{th:collective_estimator}
    For a collective estimator $\bar{W}$, the average value of the Frobenius discrepancy can be expressed in terms of just the marginal inclusion probabilities
    \begin{equation}
    \begin{aligned}
        \Eexpect \left\lVert W - \bar{W} \right\rVert_F^2 
        &= \sum_{i=1}^N \lambda_i^2 \omega_i \pi_i \left(-2 + \frac{\omega_i}{C} + \frac{\omega_i \pi_i \left(C-1\right)}{C} \right) 
        + \sum_{i=1}^N \lambda_i^2,
    \end{aligned}
    \end{equation}
    and the optimal set of inclusion probabilities and auxiliary weights has the following form in the case of $C > 1$
    \begin{align}
        \label{eq:collective_omega_pi_value}
        \omega_i = \frac{C}{1 + \pi_i \left(C-1\right)}, \qquad\qquad
        \pi_i  = \begin{cases}
            1, &\text{if} \; i \leq t, \\
             \frac{\left(n-t + \frac{u}{C-1} \right)\lambda_i}{\sum_{k=t+1}^{t+u} \lambda_k} - \frac{1}{C-1}, &\text{if} \; t < i \leq t+u, \\
            0, &\text{if} \; i > t+u, \\
        \end{cases}
    \end{align}
    for $i=1,\dots,N,$ where $t$ and $u$ are integer constants such that $0 \leq t \leq n$, and $0 \leq u \leq N-t$. 
    For $C=1$ the optimal values are $\pi_i = \omega_i = \Iindicator\lft(i \leq n \rgt).$
\end{theorem}
\begin{proof}
For $C=1$ the proof immediately follows from the Eckart–Young–Mirsky theorem~\citep{Eckart_Young_1936}  which states that truncated SVD provides the best low-rank approximation of the original matrix in terms of Frobenius discrepancy.
Since all the estimators defined by~\cref{eq:general_estimator} are low-rank approximations, it is impossible to obtain an average error which is strictly lower than the error of truncated SVD.
For the case $C > 1$, see~\cref{sec:derivation_collective_estimator}.
\end{proof}

Note that for a group consisting of a single client (i.e. $C=1$) the optimal solution coincides with top-$n$ sampling used in FedHM method.
For larger groups, similarly to the unbiased estimator, sweeping across all possible values of $t$ and $u$ is required on the server side.
Nevertheless, this search can be performed in vectorized form.

In practice, to apply \cref{th:collective_estimator}, the server can cluster the heterogeneous clients into homogeneous groups which share that same keep ratio $r$, and compute a specific sampling distribution for each group.
Such clustering is often employed by methods targeting diverse clients, cf.~\citep{mei2022resourceadaptive}.

\subsubsection{Joint Distribution and Sampling}
\label{subsec:joint_distribution}
The strategies derived in \cref{sec:unbiased_estimator,sec:collective_estimator} provide the values of marginal inclusion probabilities but do not specify the joint distribution, i.e. the co-occurrences of sampled terms.
\citet{tille_sampling_2006} presented a systematic survey of multiple sampling procedures which preserve the given inclusion probabilities.
Notably, each of the methods, while keeping the mean vector $\Eexpect \left(z_1,\dots,z_N\right) = \left(\pi_1,\dots,\pi_N\right)$ fixed, has different covariance matrices.
We follow the authors' recommendation and stick with the Conditional Poisson scheme (CPS) for our experiments.
The joint distribution achieved with this method has the maximum entropy among all distributions over samples of a size of exactly $n$ with the same marginals.
In theory, such a design should allow the model to better explore the interaction between the singular vectors during training.
The reference implementation of this sampling algorithm is provided in the accompanying package\footnote{\url{https://cran.r-project.org/package=sampling}} for \citep{tille_sampling_2006}.

\subsection{Local Training}
\label{subsec:local_training}
While the derivations in~\cref{subsec:inclusion_proba} provide an optimal (in a certain sense) estimator $\hat{W}^{\left(c\right)}$ of the full weight matrix $W$ for the client $c$ \emph{before} the local training starts on that client, they do not provide any guidance on how to optimize the received sub-model.

\topic{Auxiliary multipliers.}
As indicated in \cref{subsec:preliminaries}, we follow \citet{khodak_initialization_2022} and \citet{niu_overcoming_2023} and absorb the square root of the singular value $\lambda_i$ into both left and right singular vectors, $u'_i = \sqrt{\lambda_i} u_i, \, v'_i = \sqrt{\lambda_i} v_i$. 
During early stages of our experiments we explored the absorption of the auxiliary multipliers $\omega_i$ in the same manner. %
However, this led to unsatisfactory results.
Note that, except for the corner cases $\pi_i \in \left\{0, 1\right\}$, the multiplier $\omega_i$ is in inverse proportion to $\lambda_i$ and therefore the magnitude of their product is almost independent of $i$.
Therefore, after the product $\sqrt{\omega_i} \cdot \sqrt{\lambda_i}$ is directly `baked' into the sub-model weights, 
it is harder for the clients to distinguish between important and uninformative terms.

Instead, we treat the auxiliary multipliers as scaling factors, frozen for the current round of local training.
Hence, the effective weight matrix of the sub-model is trained locally in the factorized form $\hat{W} = U \Omega V^T$ with learnable terms $U$ and $V$.
This could introduce an issue of stale multipliers: during the course of local training the original meaning, along with the guarantees on the desired estimator properties, of the auxiliary multipliers is gradually fading away. 
However, this was not something we observe in practice and any attempts to mitigate this (e.g., with gradual decrease of the scaling factors) did not provide consistent improvement for the model performance.

\topic{Effective learning rate.}
Nevertheless, the chosen way of incorporating the auxiliary multipliers into the model still has a drawback.
Namely, the gradient of the scalar loss function $L$ w.r.t. the factors is expressed as $\frac{\partial L}{\partial U}=\frac{\partial L}{\partial \hat{W}} V \Omega, \quad \frac{\partial L}{\partial V}=\left(\frac{\partial L}{\partial \hat{W}}\right)^T U \Omega.$
This means that the auxiliary multipliers $\omega_i$ also affect the effective learning rate of the corresponding column vectors $u'_i$ and $v'_i$ when first-order optimization methods like SGD are used.
For the collective estimator it follows from \cref{eq:collective_omega_pi_value} that the values of the multipliers are bounded by the size of the group, $1 \leq \omega_i \leq C$.
However, for the unbiased estimators there is no such upper bound as $\omega_i = \pi_i^{-1}$.
Moreover, for small values of the keep ratio $r$ there necessarily should exist relatively large values of $\pi_i^{-1}$.
This is because 
\begin{equation}
\label{eq:sum_reciprocal_probas}
\resizebox{0.5\textwidth}{!}{$
    \Eexpect \sum_{j=1}^n \omega_{\Isubsample_j}
    = \Eexpect \sum_{j=1}^n \pi_{\Isubsample_j}^{-1}
    = \Eexpect \sum_{i=1}^N z_i \pi_{i}^{-1}
    = N,
$}
\end{equation}
and one needs to approximate $N$ with $n \approx rN$ randomly chosen terms.
This can lead to excessively high values of the effective learning rate.
Based on that, we adjust the nominal learning rate for each vector or, equivalently, override the value of the gradient before conducting the optimization step, $\frac{\partial L}{\partial u'_i} \leftarrow \frac{\partial L}{\partial u'_i} \cdot \min\lft(1, \frac{\tau}{\omega_i}\rgt), \quad \frac{\partial L}{\partial v'_i} \leftarrow \frac{\partial L}{\partial v'_i} \cdot \min\lft(1, \frac{\tau}{\omega_i}\rgt)$ for some threshold $\tau \geq 1$.
With this modification, the effective learning rate cannot be more than $\tau$ times higher than the nominal.
In all reported experiments we use $\tau = 10$.
Worth noting that equalizing the efficient learning rate with $\tau=1$ turned out very detrimental for the performance, as we found early in our experiments.
This coincides with the intuition behind some of adaptive optimization methods like AdaGrad~\citep{JMLR:v12:duchi11a} that perform larger updates for the parameters which are in charge of `less frequent' features.

\topic{Frobenius decay and momentum.}
As proposed in other works~\citep{khodak_initialization_2022,niu_overcoming_2023,yao_fedhm_2022} we use Frobenius weight decay (FD) during local training, i.e. we apply an additional loss function proportional to the $\left\lVert \hat{W} \right\rVert_F^2$ for each effective weight matrix.
The weight of FD in the resulting loss function is set to $\num{1e-4}$.
Additionally, confirming findings of \citet{niu_overcoming_2023}, we found it necessary to use plain SGD with momentum weight of $\num{0.9}$ during local optimization of the sub-model weights.

%% file: neurips_2024/chapters/experiments.tex
\section{Experiments}
\label{sec:experiments}

\subsection{Experimental Setup}
\topic{Datasets and models.}
To evaluate the proposed strategies, we conducted experiments on several datasets which are usually employed to test federated systems.
In case of artificial simulation of clients, we follow  \citep{hsu_nonidentical_2019} and sample the prior probabilities of classes for each client from the Dirichlet distribution $\textrm{Dir}\lft(\alpha p \rgt)$ where $p$ is the vector of class proportions in the original dataset and $\alpha$ is a hyperparameter that controls the amount of non-i.i.d.-ness in the data splits. 
Note that this notation is different from the one used in some of the recent papers, e.g.~\citep{yao_fedhm_2022,niu_federated_2022,niu_overcoming_2023}.

Since in our work we do not aim to reach new state-of-the-art in the field and instead focus on analyzing the \emph{relative} performance of the proposed strategies of spectral sharding, we have chosen to test all the methods in the presence of high data non-i.i.d.-ness between clients. 

For \emph{CIFAR-10}~\citep{Krizhevsky2009LearningML} we split the data with $\alpha=1$ and conduct experiments with a ResNet-18 model~\citep{he2015deep} with the kernel size of the first convolutional block equal to $3 \times 3$ and the normalization layer replaced with GroupNorm~\citep{pmlr-v119-hsieh20a}.
For \emph{TinyImagenet}~\citep{tinyimagenet} we use $\alpha=10$ and a compact transformer (CCT) model~\citep{hassani2022escaping}, namely CCT-6/3x2.
For \emph{CIFAR-100}~\citep{Krizhevsky2009LearningML} the two-staged Pachinko allocation method (PAM)~\citep{10.1145/1143844.1143917} is used:
for each client, at first parameters of the multinomial distribution over the twenty coarse labels are sampled with $\alpha=1$, and afterwards the distribution over the coarse-to-fine labels with $\alpha=10.$
On this dataset we train both the ResNet and CCT models described above.
We select \emph{Shakespeare}~\citep{pmlr-v54-mcmahan17a} as an example of a dataset with a natural data split over clients. 
We train a tiny transformer model with three GPT-2~\citep{radford2019language} blocks on the task of next character prediction and report the performance in terms of accuracy in accordance with prior works~\citep{reddi2021adaptive}.

For all the architectures we do not decompose the very first and very last layers of the model~\citep{khodak_initialization_2022,niu_overcoming_2023}.
However, unlike PriSM, we decompose all the affine layers inside the attention blocks, not just the fully connected ones~\citep{wang_cuttlefish_2023}.
Also, we follow \citet{niu_overcoming_2023,wang_cuttlefish_2023} instead of \citet{khodak_initialization_2022} and decompose a convolutional layer to a sequence of a regular convolution with $n$ output channels followed by a $1 \times 1$ convolution.
For image datasets we employ per-image pixel value standardization similarly to \citep{reddi2021adaptive} but do not apply random cropping.

\topic{Training.}
We train all models from scratch with cosine annealing learning rate scheduler~\citep{loshchilov2017sgdr}.
The initial value for learning rate is $0.1$ for CIFAR-10, $0.05$ for CIFAR-100 and TinyImagenet, and $0.1$ for the Shakespeare data.
The client's batch size equals 32, 64, 128, and 10 respectively.
All experiments are run with three random seeds which also affect data splitting between clients, if applicable.
Standard deviation is reported in all tables and plots based on those runs.

\topic{Federated setup.}
While Shakespeare dataset naturally contains 715 clients, we simulate 100 clients for all other datasets.
We randomly sample 10 clients for each communication rounds which results in participation ratio of $10\%$ on image datasets and about $1.4\%$ for Shakespeare.
In each communication round all participating clients train their sub-models for two local epochs.
The total number of local epochs (e.g., number of local epochs per round $\times$ number of communication rounds)  equals $\num{2000}$ for CIFAR-10, $\num{3000}$ for CIFAR-100, $\num{5000}$ for TinyImagenet and $\num{3000}$ for Shakespeare.

\topic{Baselines.}
As stated above, we focus on exploring different strategies of sampling sub-models when a particular approach of training on weak and, possibly, heterogeneous devices was chosen, namely the approach described in \cref{subsec:preliminaries}.
Therefore, we compare our presented strategies, denoted as \emph{Unbiased} and \emph{Collective}, with the \emph{Top}-$n$ sampling proposed by \citep{yao_fedhm_2022} and the \emph{PriSM} method introduced in \citep{niu_federated_2022,niu_overcoming_2023}.
We copy the value of the hyperparameter $k$ required for sampling in PriSM from the original implementation.
In detail, $k$ depends on the keep ratio $r$: $k=4$ if $r \leq 0.2$ and $k=2.5$ otherwise.

To decouple the sampling strategies themselves from other training details discussed in \cref{sec:method} we additionally introduce modifications to the PriSM strategy: motivated from our unbiased method, we train it with auxiliary multipliers to allow for (approximate) unbiasedness of the estimators, i.e. $\omega_i = \pi_i^{-1}$ (this strategy is dubbed as \emph{PriSM + Wallenius}), and clip the effective learning rate (\emph{PriSM + Wallenius + ClipLR}).
Exact computation of the mean vector for the Wallenius' distribution occurring in these two strategies is very time-consuming for the ranks $N$ of weight matrices used in practice.
For that reason we use the approximate algorithm\footnote{\url{https://cran.r-project.org/package=BiasedUrn}} from~\citep{fog_calculation_2008}.

Finally, we explore the simplest option of compensating for the missing terms of the estimator, namely, introducing a scaling factor that keeps the Frobenius norm of the estimator equal to that of the full matrix, i.e.
\begin{math}
    \omega_i^{\left(c\right)} 
    = \sqrt{\frac{\sum_{k=1}^N \lambda_k^2}{\sum_{k=1}^N z_k^{\left(c\right)} \lambda_k^2}}
\end{math}
for all $i$.
This modification is marked as \emph{+Scaled}.

All the considered methods require equal communication costs and equal amount of computation on the client side per round, since the overhead caused by the vector of auxiliary weights introduced by novel strategies is negligible.

\begin{table}[t]
\input{neurips_2024/tables/main_table}
\vspace{-5ex}
\end{table}
\begin{table}[t]
\begin{minipage}[t]{0.49\textwidth}
    \input{neurips_2024/tables/longer_CIFAR100}
\end{minipage}%
\hfill%
\begin{minipage}[t]{0.49\textwidth}

\input{neurips_2024/tables/ablation_cliplr}
\end{minipage}%
\vspace{-4ex}
\end{table}

\subsection{Results}
\subsubsection{Main Results}
We report the accuracy of different strategies in all the datasets considered in~\cref{tab:main_table}.
In our experiments we mostly explore weak clients with low keep ratio.
The results are provided for two setups with homogeneous clients ($r=0.1$ and $0.2$) and one setup with two groups, namely $60\%$ clients having a keep ratio of $0.2$ and $40\%$ of $0.4$.
Also, for reference purposes we train a conventional full model on all the clients without sharding, as well as an over-parameterized model with keep ratio $r=1$.
Similarly to what was reported by ~\citet{khodak_initialization_2022}, we observe that overparameterization can provide a better performance than usual training.

As the results demonstrate, the preservation of the Frobenius norm (\emph{+Scaled}) is the least effective and stable among the considered modifications.
It is often detrimental for the Top-$n$ strategy, and although it sometimes improves the performance of PriSM, the PriSM modification based on unbiasedness (\emph{+Wallenius+ClipLR}) is generally more successful.
In certain cases, it even achieves the best accuracy  among all models, e.g. see the results of ResNet on CIFAR-100.
However, overall the `unbiasedness' has an inconsistent effect on PriSM: it is usually beneficial in case of ResNet but is less helpful or even harmful for attention-based architectures.
As has already been noted, deeper analysis of such behavior is complicated due to unclear statistical properties of the PriSM sampling procedure and we leave this fo future work.
One possible reason is that the approximation given by Wallenius' distribution may be less precise in some cases.%

Our strategies presented in~\cref{sec:method} in most cases outperform the baselines for all datasets but Shakespeare where Top-$n$ demonstrates superior results, although the difference in performance is not significant.
As the results imply, the Collective strategy is more suitable if the keep ratio is small, while for larger $r$ the Unbiased strategy seems more preferable. 
This may be explained by the negative influence of too large auxiliary multipliers which also amplify the impact of less informative terms during the forward pass; see~\cref{subsec:local_training}.

Interestingly, in all experiments, the novel sampling strategies significantly improve the results of baselines if those had a large accuracy gap compared to the full model. 
For the Shakespeare dataset, in contrast, all the low-rank methods achieve performance close to the full model even for small values of the keep ratio. 
Therefore, it seems that the performance on Shakespeare is determined more by the properties of the data split and the corresponding federated setup than by the training strategy. 

\subsubsection{Discussion}
\topic{Explorative and exploitative strategies.}
To understand the behavior of the strategies better, we propose a tool named \emph{average normalized marginal entropy} (ANME).
For each decomposed layer with inclusion probabilities $\left(\pi_1,\dots,\pi_N\right)$ we calculate the mean entropy of the corresponding marginal Bernoulli distributions, 
\begin{math}
    \frac{1}{N} \sum_{i: \, 0 < \pi_i < 1} \left( - \pi_i \log \pi_i - \left(1 - \pi_i\right) \log \left(1 - \pi_i\right)\right).
\end{math}
It is easy to show that the minimum value of this quantity equals 0 and is achieved by the deterministic Top-$n$ strategy.
The maximum value is achieved when all probabilities are equal, i.e. $\pi_i = \frac{n}{N}$.
Therefore, we normalize the mean entropy from above by dividing it by the mean entropy of the uniform inclusion probabilities.
To finish the computation of ANME, we average this normalized entropy across all sharded layers in the network.

Intuitively, low values of ANME mean that the server tends to send the same terms to all the clients, and we name such strategies \emph{exploitative}.
On the contrary, high value of ANME shows that the strategy is \emph{explorative} and the selection of terms for each client is `more random'.
We observe the same qualitative behaviour in all our experiments (see~\cref{fig:avg_normalized_marginal_entropy}): PriSM is the most explorative strategy, while the modification based on unbiasedness turns it into the most exploitative among randomized methods. 
The Unbiased and Collective strategies are between these two extreme types of behaviour, and the latter is usually more exploitative than the former one.
Based on this observation, we check if training with more communication rounds can help the explorative strategies to perform better and report the results in~\cref{tab:train_longer_cifar100}.
Indeed, all methods improve with more training and the largest improvement is for PriSM, which is the most explorative. We also observe that our collective strategy overtakes our modified PriSM baseline and achieves the best performance.

\topic{Influence of learning rate clipping.}
In all the experiments we found that clipping the effective learning rate is necessary for our strategies.
\cref{tab:ablation_cliplr} illustrates this observation by using models trained on CIFAR-100.
Notably, \cref{tab:main_table} demonstrates that, empirically, such clipping is not always beneficial for the modified PriSM method.
This may be explained by the fact that `unbiased' PriSM becomes extremely exploitative.
In this case, terms which could produce too large auxiliary multipliers are very unlikely to be sampled in practice, resulting in the large variance of the left-hand side of~\cref{eq:sum_reciprocal_probas}.

\topic{More results.}
Please refer to the~\cref{sec:additional_discussion} for additional results on model convergence, post-client updates and influence of data heterogeneity.

\begin{figure}[t]
\centering
\begin{subfigure}[t]{0.49\textwidth}
    \centering
    \includegraphics[width=\textwidth]{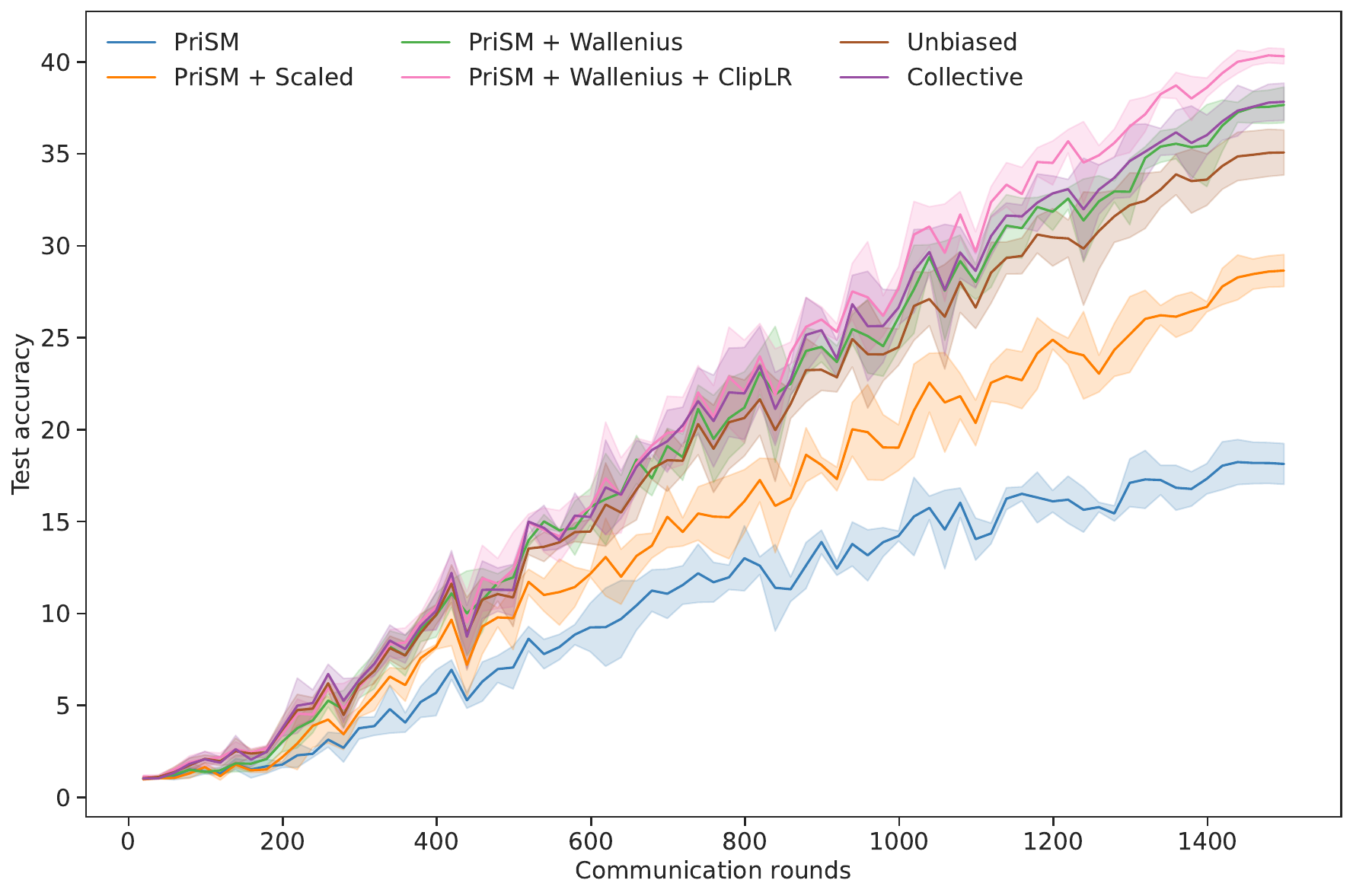}
    \vspace{-3ex}
    \caption{Accuracy of the global ResNet model on the test subset of CIFAR-100.}
    \label{fig:communication_efficiency}
\end{subfigure}
\hfill
\begin{subfigure}[t]{0.49\textwidth}
    \centering
    \includegraphics[width=\textwidth]{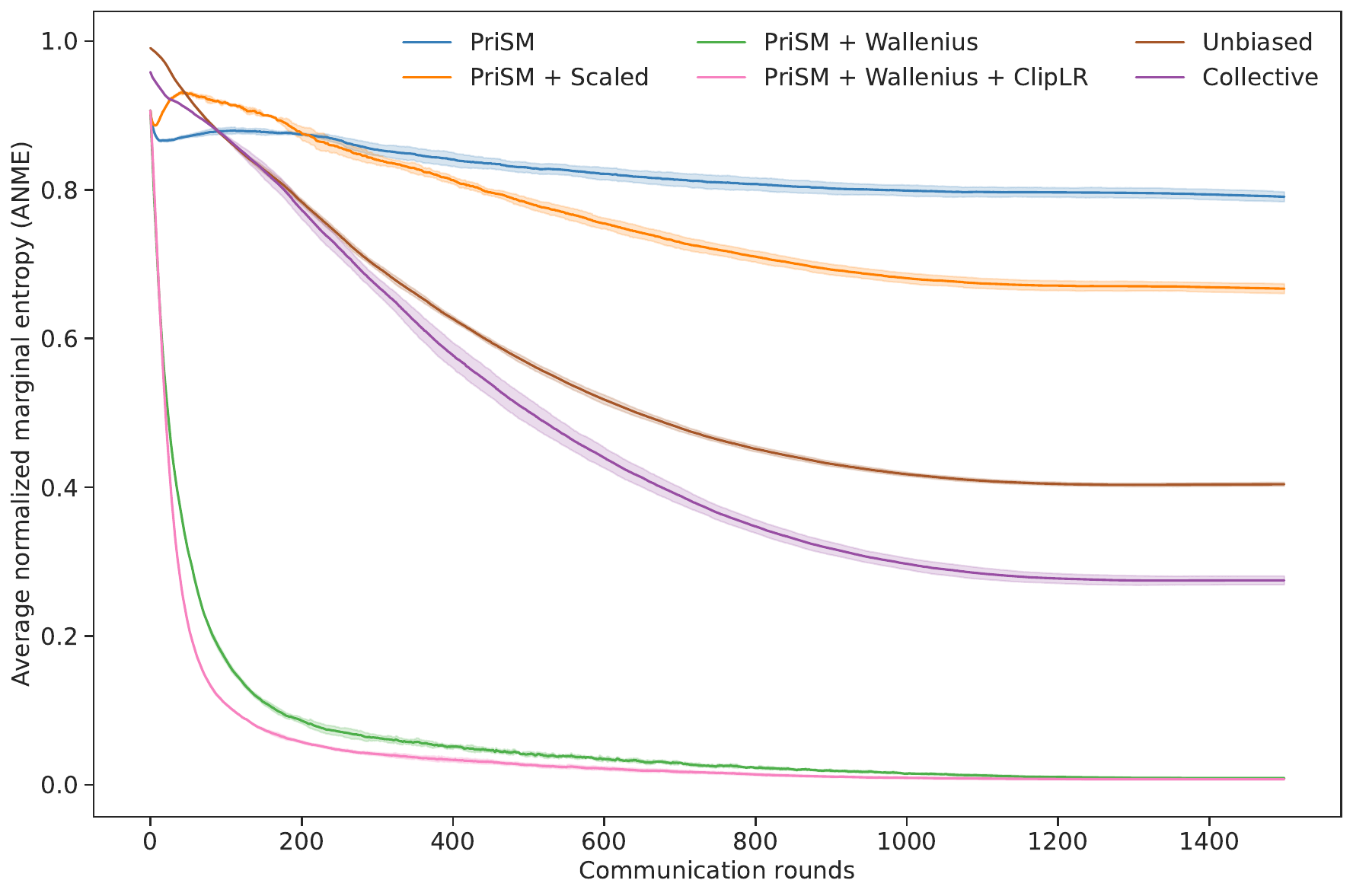}
    \vspace{-3ex}
    \caption{Average normalized marginal entropy (ANME) shows how diverse each strategy is.}
    \label{fig:avg_normalized_marginal_entropy}
\end{subfigure}
\caption{%
    \topic[0mm]{Communication efficiency.}
    The original PriSM method is too explorative (high ANME), while our `unbiased' modification (\emph{+Wallenius}) makes it the most exploitative strategy and allows to achieve the best performance in some experiments with limited computational budget.
}
\label{fig:main_figure}
\vspace{-4ex}
\end{figure}

%% file: neurips_2024/tables/main_table.tex
\newcommand{\leftspace}{\phantom{ab}}
\caption{%
    \topic[0mm]{Accuracy achieved by different strategies under limited computational budget}.
    Our \emph{Unbiased} and \emph{Collective} strategies outperform the vanilla \emph{Top}-$n$ and \emph{PriSM} baselines on all datasets except for Shakespeare, although the gap is not significant for that dataset. 
    In addition, the modifications proposed for local training allow PriSM to significantly improve its performance for ResNet architecture, and sometimes even surpass  other methods in the current setting.
}
\label{tab:main_table}
\begin{center}
\begin{small}
\begin{sc}
\resizebox{0.77\textwidth}{!}{ \begin{tabular}{@{}lccccc}
\toprule
\multirow{2}{*}{\leftspace Strategy}%
& CIFAR-10%
& TinyImagenet%
& \multicolumn{2}{c}{CIFAR-100} %
& Shakespeare\\%
& ResNet%
& CCT%
& ResNet%
& CCT%
& Transformer\\%
\midrule%
\textit{Keep ratio} $r = \num{0.1}$\\%
\leftspace Top-$n$   & ${ 67.44 }_{ \pm{ 2.34 } }$ & ${ 34.29 }_{ \pm{ 0.12 } }$ & ${ 21.13 }_{ \pm{ 2.01 } }$ & ${ 47.34 }_{ \pm{ 1.10 } }$ & $\mathbf{ 48.83 }_{ \pm{ 0.45 } }$ \\%
\leftspace Top-$n$ + Scaled   & ${ 63.01 }_{ \pm{ 0.65 } }$ & ${ 22.29 }_{ \pm{ 0.76 } }$ & ${ 14.13 }_{ \pm{ 0.76 } }$ & ${ 33.96 }_{ \pm{ 0.77 } }$ & ${ 36.14 }_{ \pm{ 0.77 } }$ \\%
\leftspace PriSM   & ${ 71.15 }_{ \pm{ 3.77 } }$ & ${ 33.61 }_{ \pm{ 0.88 } }$ & ${ 18.45 }_{ \pm{ 1.31 } }$ & ${ 44.60 }_{ \pm{ 0.90 } }$ & ${ 46.91 }_{ \pm{ 1.36 } }$ \\%
\leftspace PriSM + Scaled   & ${ 76.03 }_{ \pm{ 5.00 } }$ & ${ 35.10 }_{ \pm{ 0.39 } }$ & ${ 28.65 }_{ \pm{ 1.07 } }$ & ${ 42.14 }_{ \pm{ 0.88 } }$ & ${ 43.34 }_{ \pm{ 3.17 } }$ \\%
\leftspace PriSM + Wallenius   & ${ 77.53 }_{ \pm{ 1.25 } }$ & ${ 37.43 }_{ \pm{ 0.87 } }$ & ${ 37.75 }_{ \pm{ 1.04 } }$ & ${ 42.04 }_{ \pm{ 1.11 } }$ & ${ 45.60 }_{ \pm{ 0.66 } }$ \\%
\leftspace PriSM + Wallenius + ClipLR   & ${ 82.42 }_{ \pm{ 0.35 } }$ & ${ 35.43 }_{ \pm{ 0.91 } }$ & $\mathbf{ 40.36 }_{ \pm{ 0.49 } }$ & ${ 35.50 }_{ \pm{ 1.24 } }$ & ${ 48.24 }_{ \pm{ 0.73 } }$ \\%
\leftspace Unbiased   & ${ 80.57 }_{ \pm{ 1.39 } }$ & ${ 37.57 }_{ \pm{ 0.17 } }$ & ${ 35.12 }_{ \pm{ 1.51 } }$ & ${ 45.01 }_{ \pm{ 1.07 } }$ & ${ 48.23 }_{ \pm{ 0.35 } }$ \\%
\leftspace Collective   & $\mathbf{ 82.59 }_{ \pm{ 0.27 } }$ & $\mathbf{ 38.72 }_{ \pm{ 0.86 } }$ & ${ 37.83 }_{ \pm{ 1.24 } }$ & $\mathbf{ 49.93 }_{ \pm{ 1.77 } }$ & ${ 48.64 }_{ \pm{ 0.50 } }$ \\%
\textit{Keep ratio} $r = \num{0.2}$\\%
\leftspace Top-$n$   & ${ 78.37 }_{ \pm{ 0.30 } }$ & ${ 38.37 }_{ \pm{ 0.42 } }$ & ${ 35.97 }_{ \pm{ 3.19 } }$ & ${ 51.70 }_{ \pm{ 1.24 } }$ & $\mathbf{ 51.21 }_{ \pm{ 0.50 } }$ \\%
\leftspace Top-$n$ + Scaled   & ${ 80.75 }_{ \pm{ 1.06 } }$ & ${ 34.99 }_{ \pm{ 0.33 } }$ & ${ 35.29 }_{ \pm{ 2.02 } }$ & ${ 47.68 }_{ \pm{ 0.70 } }$ & ${ 45.82 }_{ \pm{ 0.38 } }$ \\%
\leftspace PriSM   & ${ 82.59 }_{ \pm{ 1.09 } }$ & ${ 38.66 }_{ \pm{ 0.28 } }$ & ${ 33.76 }_{ \pm{ 0.76 } }$ & ${ 52.04 }_{ \pm{ 1.62 } }$ & ${ 49.63 }_{ \pm{ 0.50 } }$ \\%
\leftspace PriSM + Scaled   & ${ 81.63 }_{ \pm{ 3.52 } }$ & ${ 40.30 }_{ \pm{ 0.48 } }$ & ${ 40.86 }_{ \pm{ 1.08 } }$ & ${ 51.64 }_{ \pm{ 1.30 } }$ & ${ 48.30 }_{ \pm{ 0.97 } }$ \\%
\leftspace PriSM + Wallenius   & ${ 82.15 }_{ \pm{ 0.14 } }$ & ${ 40.47 }_{ \pm{ 0.93 } }$ & ${ 45.11 }_{ \pm{ 0.63 } }$ & ${ 45.22 }_{ \pm{ 0.86 } }$ & ${ 45.35 }_{ \pm{ 0.50 } }$ \\%
\leftspace PriSM + Wallenius + ClipLR   & ${ 84.60 }_{ \pm{ 1.24 } }$ & ${ 38.13 }_{ \pm{ 0.62 } }$ & ${ 46.73 }_{ \pm{ 1.30 } }$ & ${ 38.54 }_{ \pm{ 1.48 } }$ & ${ 49.44 }_{ \pm{ 0.39 } }$ \\%
\leftspace Unbiased   & ${ 84.33 }_{ \pm{ 1.14 } }$ & $\mathbf{ 41.28 }_{ \pm{ 0.31 } }$ & ${ 46.34 }_{ \pm{ 0.02 } }$ & ${ 53.39 }_{ \pm{ 0.81 } }$ & ${ 50.12 }_{ \pm{ 0.89 } }$ \\%
\leftspace Collective   & $\mathbf{ 85.00 }_{ \pm{ 0.34 } }$ & ${ 41.12 }_{ \pm{ 0.11 } }$ & $\mathbf{ 47.74 }_{ \pm{ 1.15 } }$ & $\mathbf{ 54.31 }_{ \pm{ 0.27 } }$ & ${ 50.05 }_{ \pm{ 0.44 } }$ \\%
\multicolumn{6}{@{}l}{\textit{60\% clients have keep ratio $r=0.2$ and 40\% clients have $r=0.4$}}\\%
\leftspace Top-$n$ & ${ 84.43 }_{ \pm{ 0.90 } }$ & ${ 36.89 }_{ \pm{ 0.14 } }$ & ${ 41.34 }_{ \pm{ 1.36 } }$ & ${ 53.69 }_{ \pm{ 0.71 } }$ & $\mathbf{ 51.35 }_{ \pm{ 0.89 } }$ \\%
\leftspace Top-$n$ + Scaled   & ${ 85.24 }_{ \pm{ 0.19 } }$ & ${ 37.12 }_{ \pm{ 0.30 } }$ & ${ 42.24 }_{ \pm{ 0.80 } }$ & ${ 47.99 }_{ \pm{ 1.26 } }$ & ${ 50.11 }_{ \pm{ 0.70 } }$ \\%
\leftspace PriSM   & ${ 82.37 }_{ \pm{ 0.62 } }$ & ${ 31.25 }_{ \pm{ 0.21 } }$ & ${ 23.04 }_{ \pm{ 2.01 } }$ & ${ 37.47 }_{ \pm{ 0.26 } }$ & ${ 49.75 }_{ \pm{ 0.03 } }$ \\%
\leftspace PriSM + Scaled   & ${ 82.75 }_{ \pm{ 0.14 } }$ & ${ 34.86 }_{ \pm{ 0.41 } }$ & ${ 28.41 }_{ \pm{ 2.12 } }$ & ${ 42.45 }_{ \pm{ 0.90 } }$ & ${ 49.69 }_{ \pm{ 0.22 } }$ \\%
\leftspace PriSM + Wallenius   & ${ 78.71 }_{ \pm{ 1.18 } }$ & ${ 34.83 }_{ \pm{ 1.09 } }$ & ${ 34.03 }_{ \pm{ 2.02 } }$ & ${ 34.23 }_{ \pm{ 1.00 } }$ & ${ 38.21 }_{ \pm{ 2.02 } }$ \\%
\leftspace PriSM + Wallenius + ClipLR   & ${ 82.95 }_{ \pm{ 0.84 } }$ & ${ 39.31 }_{ \pm{ 0.17 } }$ & ${ 32.88 }_{ \pm{ 0.66 } }$ & ${ 40.07 }_{ \pm{ 0.39 } }$ & ${ 49.45 }_{ \pm{ 0.85 } }$ \\%
\leftspace Unbiased   & ${ 85.09 }_{ \pm{ 0.52 } }$ & $\mathbf{ 40.78 }_{ \pm{ 0.18 } }$ & $\mathbf{ 44.61 }_{ \pm{ 2.43 } }$ & $\mathbf{ 54.16 }_{ \pm{ 0.59 } }$ & ${ 51.05 }_{ \pm{ 0.63 } }$ \\%
\leftspace Collective   & $\mathbf{ 86.11 }_{ \pm{ 0.66 } }$ & ${ 39.90 }_{ \pm{ 0.43 } }$ & ${ 43.74 }_{ \pm{ 1.87 } }$ & ${ 52.52 }_{ \pm{ 0.93 } }$ & ${ 50.74 }_{ \pm{ 1.15 } }$ \\%
\midrule
\textit{Keep ratio} $r=1$ & ${ 89.54 }_{ \pm{ 0.49 } }$ & ${ 41.28 }_{ \pm{ 0.65 } }$ & ${ 63.45 }_{ \pm{ 0.34 } }$ & ${ 59.25 }_{ \pm{ 1.05 } }$ & ${ 52.10 }_{ \pm{ 1.22 } }$ \\%
\textit{No sharding}  & ${ 87.46 }_{ \pm{ 0.31 } }$ & ${ 38.68 }_{ \pm{ 0.16 } }$ & ${ 63.35 }_{ \pm{ 0.68 } }$ & ${ 60.05 }_{ \pm{ 0.79 } }$ & ${ 52.89 }_{ \pm{ 0.10 } }$ \\%
\bottomrule
\end{tabular} }
\end{sc}
\end{small}
\end{center}

%% file: neurips_2024/tables/longer_CIFAR100.tex
\caption{
    \topic[0mm]{Longer training on CIFAR-100.}
    Despite our `unbiased' modification of PriSM  demonstrated the best accuracy in case of limited computation budget, the more explorative Collective strategy closes the gap in performance if the number of communication rounds is increased.
}
\label{tab:train_longer_cifar100}
\vspace{-1ex}
\begin{center}
\begin{small}
\begin{sc}
\resizebox{0.88\textwidth}{!}{ \begin{tabular}{lcc}
\toprule
\multirow{2}{*}{Strategy}%
& \multicolumn{2}{c}{\# communication rounds}\\%
& \num{1500} & \num{5000}\\%
\midrule%
Top-$n$ & ${ 21.13 }_{ \pm{ 2.01 } }$ & ${ 32.22 }_{ \pm{ 1.74 } }$ \\%
Top-$n$ + Scaled & ${ 14.13 }_{ \pm{ 0.76 } }$ & ${ 39.68 }_{ \pm{ 0.32 } }$ \\%
PriSM & ${ 18.45 }_{ \pm{ 1.31 } }$ & ${ 51.85 }_{ \pm{ 1.25 } }$ \\%
PriSM + Scaled & ${ 28.65 }_{ \pm{ 1.07 } }$ & ${ 57.08 }_{ \pm{ 0.45 } }$ \\%
PriSM + Wallenius & ${ 37.75 }_{ \pm{ 1.04 } }$ & ${ 59.81 }_{ \pm{ 0.37 } }$ \\%
PriSM + Wallenius + ClipLR & $\mathbf{ 40.36 }_{ \pm{ 0.49 } }$ & ${ 60.02 }_{ \pm{ 0.50 } }$ \\%
Unbiased & ${ 35.12 }_{ \pm{ 1.51 } }$ & ${ 59.44 }_{ \pm{ 0.08 } }$ \\%
Collective & ${ 37.83 }_{ \pm{ 1.24 } }$ & $\mathbf{ 60.26 }_{ \pm{ 0.31 } }$ \\%
\bottomrule
\end{tabular} }
\end{sc}
\end{small}
\end{center}

%% file: neurips_2024/tables/ablation_cliplr.tex
\newcommand{\leftspace}{\phantom{ab}}
\caption{%
    \topic[0mm]{Influence of the clipped learning rate.}
    In our experiments, we find that clipping the effective learning rate is beneficial for the Unbiased strategy in case of all architectures and values of the keep ratio $r$.
    Without clipping  the performance drops consistently.
}
\label{tab:ablation_cliplr}
\vspace{-1ex}
\begin{center}
\begin{small}
\begin{sc}
\resizebox{\textwidth}{!}{ 
\begin{tabular}{@{}lcc}
\toprule
\leftspace Strategy
& ResNet
& CCT\\%
\midrule
\textit{Keep ratio} $r=\num{0.1}$\\%
\leftspace Unbiased & ${ 35.12 }_{ \pm{ 1.51 } }$  & ${ 45.01 }_{ \pm{ 1.07 } }$\\%
\leftspace Unbiased w/o ClipLR & ${ 30.14 }_{ \pm{ 1.58 } }$ & ${ 32.46 }_{ \pm{ 0.71 } }$ \\%
\textit{Keep ratio} $r=\num{0.2}$\\%
\leftspace Unbiased &  ${ 46.34 }_{ \pm{ 0.02 } }$ &   ${ 53.39 }_{ \pm{ 0.81 } }$\\%
\leftspace Unbiased w/o ClipLR & ${ 43.11 }_{ \pm{ 1.02 } }$ &  ${ 46.67 }_{ \pm{ 0.85 } }$\\%
\bottomrule
\end{tabular} 
}
\end{sc}
\end{small}
\end{center}

%% file: neurips_2024/chapters/conclusion.tex
\section{Conclusion}
\label{sec:conclusion}
We presented two novel sampling strategies for spectral model sharding, grounded as solutions to specific optimization problems. 
Alongside them, we presented techniques that facilitate local training with such methods. %
As shown, in a number of cases these techniques can also significantly improve the performance of the strategies proposed earlier in the literature.
Nonetheless, our strategies demonstrate superior performance on a number of commonly used datasets in the presence of high data heterogeneity between clients.
As a downside, in certain cases their learning curve is less steep than that of the baselines due to their more explorative nature.
We leave the improvement of the convergence speed  for future work.

%% file: neurips_2024/appendix/additional_discussion.tex
\begin{figure}[t]
\centering
\includegraphics[width=\textwidth]{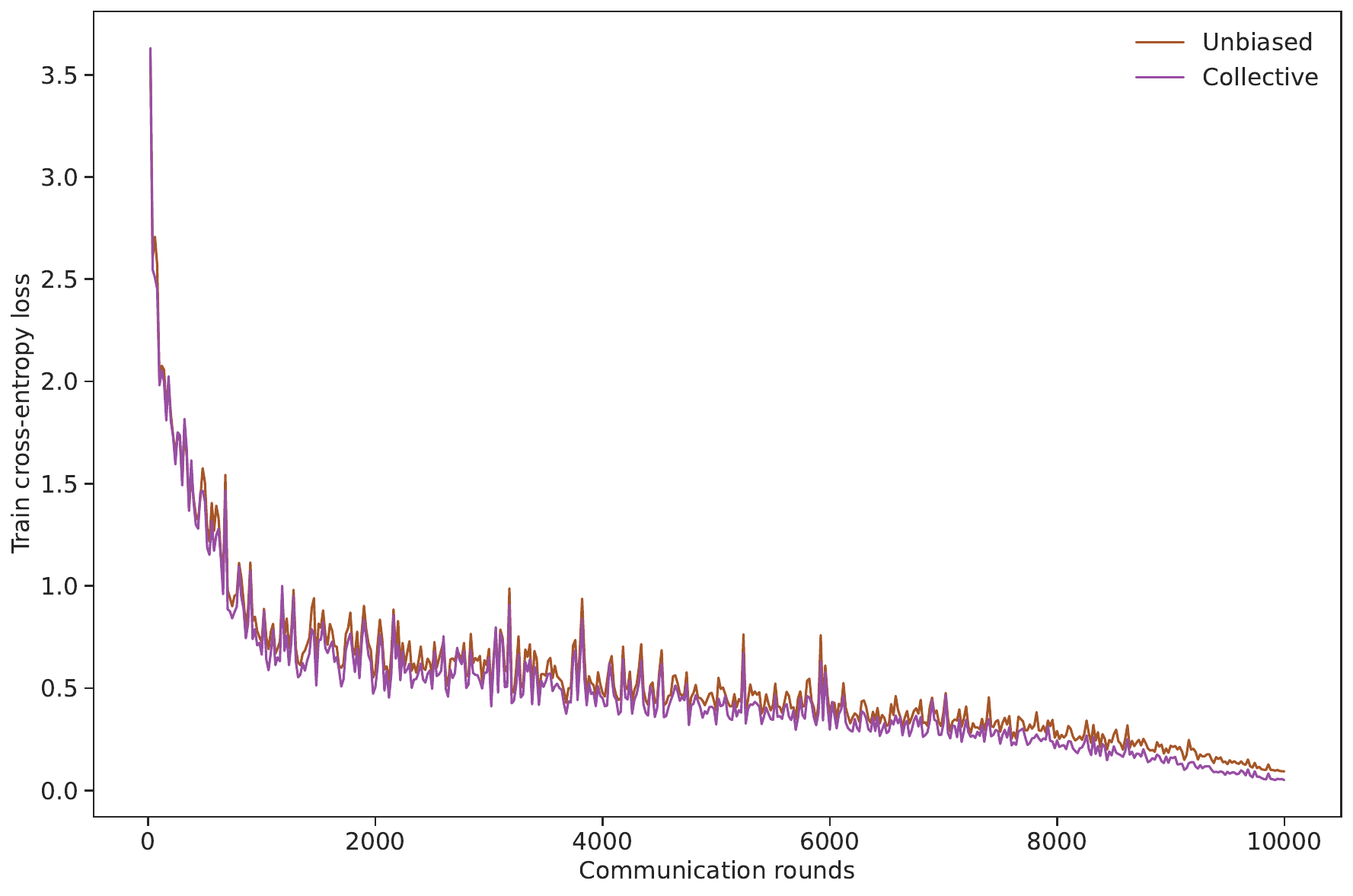}
\vspace{-3ex}
\caption{%
    \topic[0mm]{Convergence analysis.}
    When being trained longer, the proposed strategies demonstrate the decrease of the cross-entropy loss of the global model on the training set.
    Unbiased strategy reaches the train accuracy of $\num{97.0}\%$, and Collective strategy achieves $\num{98.4}\%$.
    This serves as empirical evidence of convergence for our method.
}
\label{fig:convergence_analysis}
\vspace{1ex}
\nextfloat
\centering
\begin{subfigure}[t]{0.49\textwidth}
    \centering
    \includegraphics[width=\textwidth]{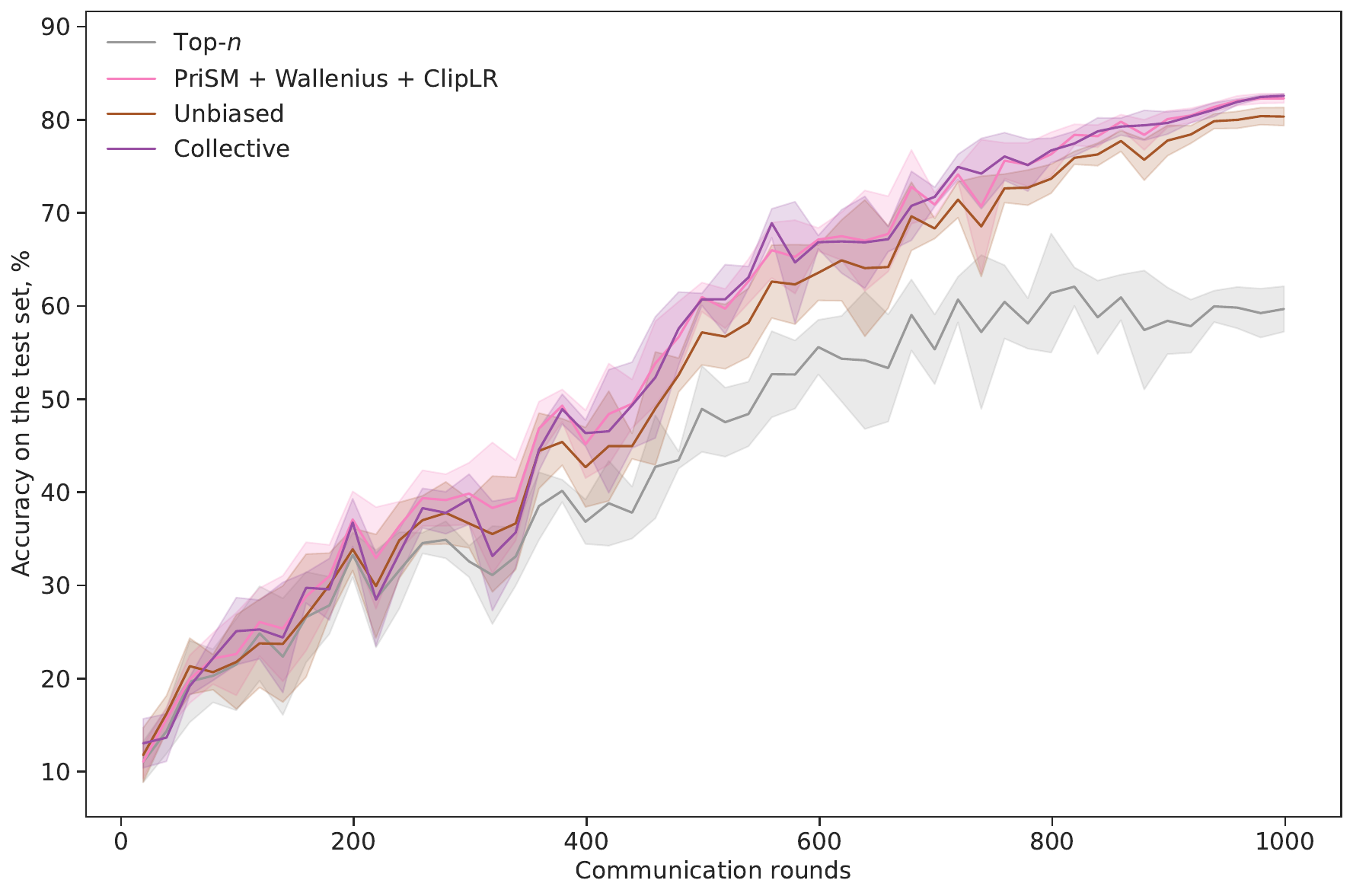}
    \subcaption{High data heterogeneity $\alpha = \num{1}.$}
\end{subfigure}
\hfill
\begin{subfigure}[t]{0.49\textwidth}
    \centering
    \includegraphics[width=\textwidth]{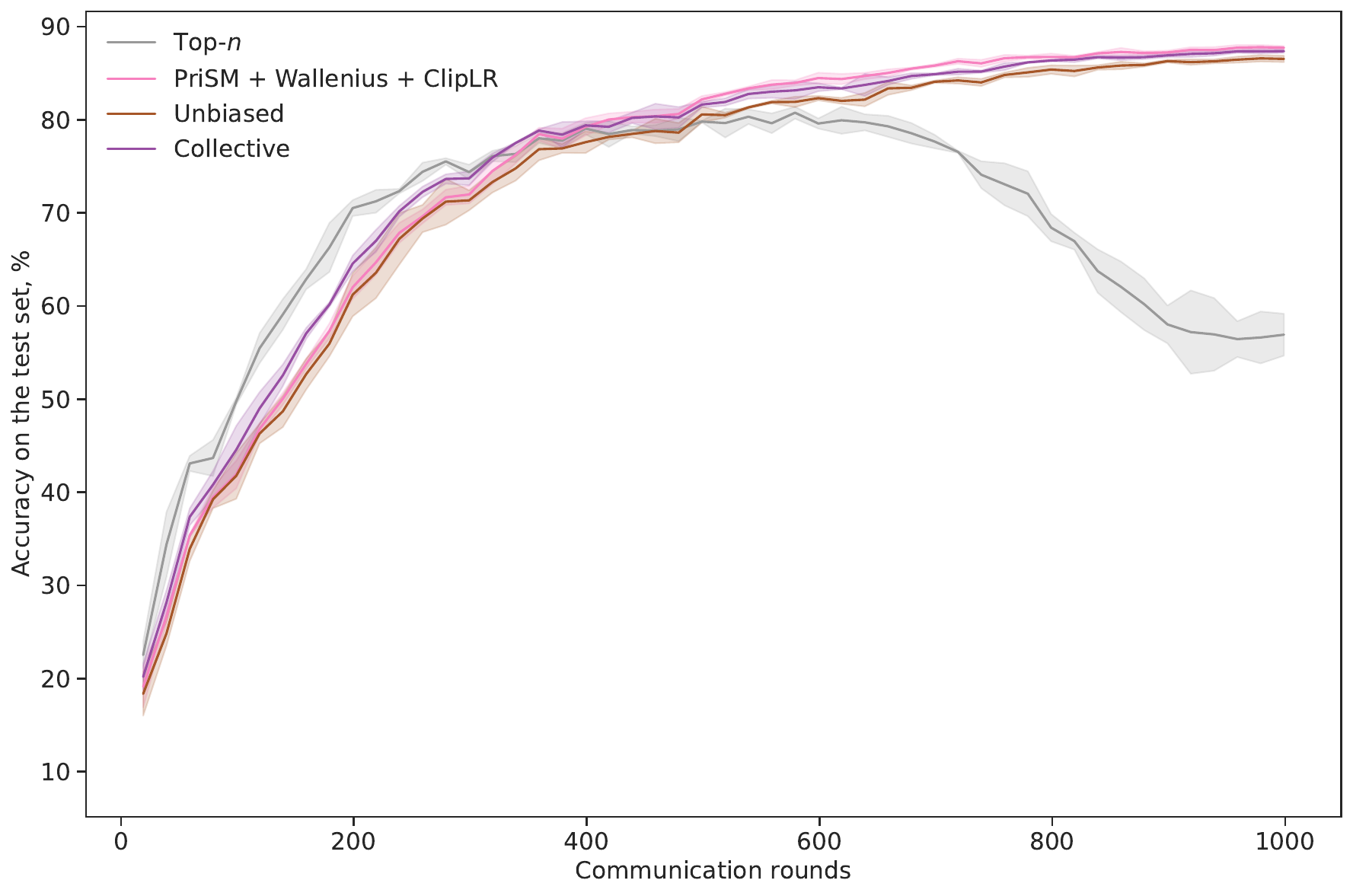}
    \subcaption{Low data heterogeneity $\alpha = \num{10}.$}
\end{subfigure}
\caption{%
    \topic[0mm]{Impact of data heterogeneity on communication efficiency.}
    When data distribution between clients is closer to i.i.d, the most exploitative Top-$n$ strategy demonstrates the best training speed in the beginning, however it overfits soon. 
    For more heterogeneous data, strategies with exploration achieve much better performance.
}
\label{fig:data_hetero_and_communication}
\end{figure}

\section{Additional Discussion}
\label{sec:additional_discussion}

\topic{Convergence analysis.}
It is rather involved to provide the theoretical proof of convergence for any of spectral sharding methods. 
In practice, we observe that the models trained in this manner successfully converge even in the presence of high data heterogeneity between clients. 
As an example, \cref{fig:convergence_analysis} demonstrates the cross-entropy loss on the train set for the ResNet model with keep ratio $r = 0.1$  trained on CIFAR-10 with non-i.i.d.-ness parameter $\alpha = 1$.
The loss was computed at the end of each communication round, the number of rounds was intentionally increased up to \num{10000} for this experiment. 
It is easy to see that the loss successfully decreases as the training progresses. 
Therefore, given these empirical observations, we did not focus on the formal convergence analysis.
Having said that, it is an interesting direction for future work, as it can potentially uncover further improvements to our approach.

\topic{Impact of data heterogeneity.}
Throughout our paper, we evaluate all the strategies in the presence of a relatively high degree of data heterogeneity between clients. 
E.g., note that $\alpha = 1$ for CIFAR-10 roughly corresponds to only one or two distinct class labels present for each client, and $\alpha = 10$ for TinyImagenet means that each client only has access to 20-30\% of all class labels.

We demonstrate the effect of data heterogeneity on the communication efficiency of the best performing strategies on CIFAR-10  in~\cref{fig:data_hetero_and_communication}.
When clients’ data distributions are closer to i.i.d., i.e. $\alpha = 10$, Top-$n$ strategy is the most efficient in the beginning, however it overfits soon. 
For more heterogeneous data, i.e. $\alpha = 1$, the proposed modification of PriSM and the novel Collective strategy are the most efficient and deliver the best accuracy.

Note that the novel strategies developed in our work do not use any knowledge of data distribution between clients.
Neither they rely on the presence of a public dataset on the server side. 
Nevertheless, we believe that incorporating such information can be beneficial for the Collective strategy and leave this for future work.

\topic{Influence of the joint distribution.}
To test the sensitivity of the results to the choice of the joint distribution of the sampled indices while keeping the marginal distributions frozen, we compare the selected CPS sampling against two other methods, Brewer's~\citep{brewer_simple_1975} and Minimum support~\citep{tille_sampling_2006}.
The advantage of those methods is that they allow for significantly faster sampling than CPS.
Perhaps surprisingly, we do not notice any significant difference in our experiments, as reported in~\cref{tab:ablation_sampling}.
We suggest that the effect of high entropy which is characteristic for CPS may be shown on more large-scale problems and leave this for future work.
\begin{wrapfigure}{l}{0.5\textwidth}
\centering
\includegraphics[width=0.5\textwidth]{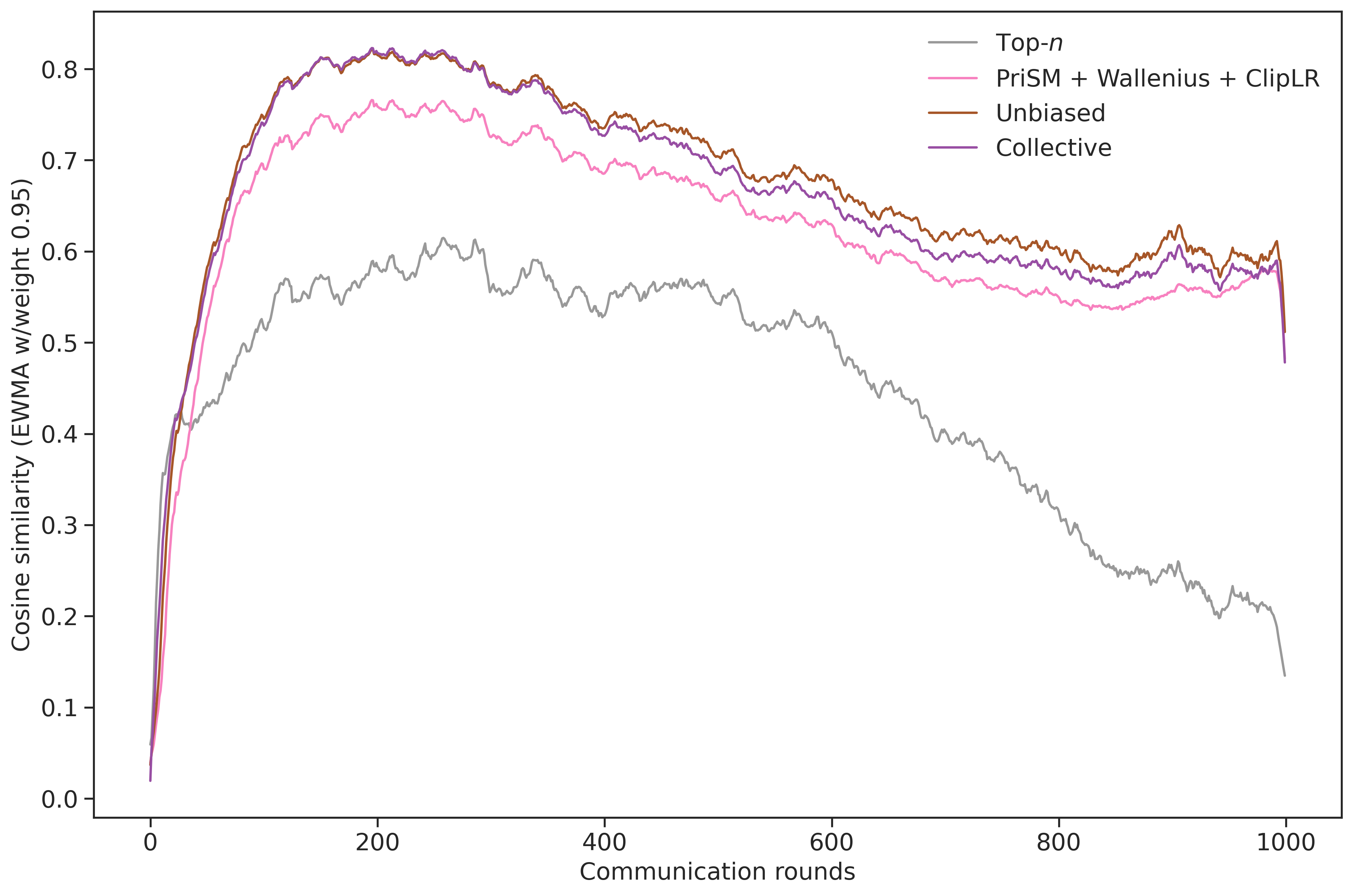}
\vspace{-1ex}
\caption{\topic[0mm]{Comparison with FedAvg weight updates.}
    For ResNet model trained on CIFAR-10, updates provided by the Top-$n$ strategy significantly deviate from those of FedAvg method. This correlates with the worse performance in this experiment.
}
\label{fig:post_client_updates}
\end{wrapfigure}

\topic{Post-client updates.}
To analyze the post-gradient updates, we followed the optimization path $ \{ \theta_t \}_{t=1}^{T} $ of the server model under each strategy with $T$ communication rounds and calculated the cosine similarity between the server model update $\Delta_{sharding\_strategy} \theta_t$ under that strategy and $\Delta_{FedAvg} \theta_t$ under FedAvg method with full-size model on each client. 
Subsequently, the model weights were updated according to the selected strategy: $\theta_{t+1} = \theta_t + \Delta_{sharding\_strategy} \theta_t. $ 
The results of these experiments are provided in~\cref{fig:post_client_updates}.
As shown, there is a notable connection between the relative performance of each strategy and its “alignment” with FedAvg updates. 
The deterministic Top-$n$ strategy lags behind its randomized counterparts in terms of accuracy in the considered setting (see~\cref{tab:main_table}), and a similar observation is made for its update directions.

\begin{table}[t]
\begin{minipage}[b]{0.49\textwidth}
    \input{neurips_2024/tables/ablation_sampling}
\end{minipage}%
\hfill%
\begin{minipage}[b]{0.49\textwidth}
    \input{neurips_2024/tables/schatten}
\end{minipage}%
\vspace{-4ex}
\end{table}

\topic{Schatten norm.}
Throughout this paper we used Frobenius norm to upper-bound the operator norm, see~\cref{eq:operator_and_frobenius_norm}.
While this choice allowed us to derive a closed-form solution which enables fast calculation on the server side, this upper bound is loose.
Better approximation can be achieved by using Schatten $p$-norm with $p > 2$.
Schatten norm  is defined as 
\begin{math}
    \left\lVert W \right\rVert_p = \left(\lambda_1^p + \dots + \lambda_N^p\right)^{\frac{1}{p}},
\end{math}
where $\lambda_1, \dots, \lambda_N$ are singular values of the matrix $W$.
Frobenius norm is an example of a Schatten $p$-norm with $p=2$.

For the \emph{unbiased} matrix estimators of the kind specified in~\cref{eq:general_estimator}, the average Schatten norm can be upper-bounded with the help of Jensen's inequality, since $\left(\cdot\right)^{\frac{1}{p}}$ is a concave function for $p > 1$
\begin{equation}\begin{aligned}
    \Eexpect_{p\left(z\right)}  \left\lVert W - \hat{W} \right\rVert_p
    &= \Eexpect_{p\left(z\right)} \left(\sum_{i=1}^N \lambda_i^p \left\lvert \frac{z_i}{\pi_i} - 1 \right\rvert^p\right)^{\frac{1}{p}} \\
    &\leq \left(\Eexpect_{p\left(z\right)} \sum_{i=1}^N \lambda_i^p \left\lvert \frac{z_i}{\pi_i} - 1 \right\rvert^p  \right)^{\frac{1}{p}} \\
    &= \left( \sum_{i=1}^N \lambda_i^p \left[1 - \pi_i + \pi_i \left(\pi_i^{-1} -1 \right)^p \right] \right)^{\frac{1}{p}}.
\end{aligned}\end{equation}
Unfortunately, we were not able to derive a closed-form solution for large values of $p$ in the same way as we proved~\cref{th:unbiased_estimator}.
Therefore, we opted for  numerical constrained optimization on the server side to find the optimal values of inclusion probabilities. 
As demonstrated in~\cref{tab:schatten}, for $p=4$ we did not find any significant improvement over our strategies presented in~\cref{sec:method}.
For larger $p$, the optimization procedure turned out unstable, and the results were unsatisfactory.

\topic{Computational resources.}
All our experiments were conducted with a single GPU and required not more than 10 Gb VRAM.
We used a simulation of the federated learning environment with all participating clients being trained in a sequence on the same GPU.

%% file: neurips_2024/tables/ablation_sampling.tex
\caption{%
    \topic[0mm]{Influence of the joint distribution.}
    We have not found the evidence of the advantage of maximum entropy distribution (CPS) over other methods for our tasks.
    On CIFAR-100, all three sampling methods perform similarly provided that marginal inclusion probabilities are preserved.
}
\label{tab:ablation_sampling}
\begin{center}
\begin{small}
\begin{sc}
\resizebox{\textwidth}{!}{ 
\begin{tabular}{lclccc}
\toprule
Model
& \footnotesize {\shortstack{Keep \\ ratio}}%
& Strategy%
& CPS
& Brewer
& MinSupport \\%
\midrule%
\multirow{4}{*}{ResNet} & \multirow{2}{*}{0.1} & Unbiased & ${ 35.12 }_{ \pm{ 1.51 } }$ & ${ 35.67 }_{ \pm{ 1.73 } }$ & ${ 35.19 }_{ \pm{ 1.72 } }$ \\%
&& Collective & ${ 37.83 }_{ \pm{ 1.24 } }$ & ${ 37.55 }_{ \pm{ 1.92 } }$ & ${ 37.89}_{ \pm{ 1.91 } }$ \\%
\cmidrule{2-6}
& \multirow{2}{*}{0.2} & Unbiased & ${ 46.34 }_{ \pm{ 0.12 } }$ & ${ 46.94 }_{ \pm{ 1.12 } }$ & ${ 45.71 }_{ \pm{ 1.80 } }$ \\%
&& Collective & ${ 47.74 }_{ \pm{ 1.15 } }$ & ${ 43.55 }_{ \pm{ 1.06 } }$ & ${ 46.13 }_{ \pm{ 1.23 } }$ \\%
\midrule
\multirow{4}{*}{CCT} & \multirow{2}{*}{0.1} & Unbiased & ${ 45.01 }_{ \pm{ 1.07 } }$ & ${ 45.82 }_{ \pm{ 0.70 } }$ & ${ 45.34 }_{ \pm{ 1.05 } }$ \\%
&& Collective & ${ 49.93 }_{ \pm{ 1.77 } }$ & ${ 49.51 }_{ \pm{ 1.63 } }$ & ${ 50.65 }_{ \pm{ 1.62 } }$ \\%
\cmidrule{2-6}
& \multirow{2}{*}{0.2} & Unbiased & ${ 53.39 }_{ \pm{ 0.81 } }$ & ${ 53.29 }_{ \pm{ 1.46 } }$ & ${ 52.94 }_{ \pm{ 1.01 } }$ \\%
&& Collective & ${ 54.31 }_{ \pm{ 0.27 } }$ & ${ 53.74 }_{ \pm{ 1.06 } }$ & ${ 52.98 }_{ \pm{ 0.52 } }$ \\%
\bottomrule
\end{tabular} 
}
\end{sc}
\end{small}
\end{center}

%% file: neurips_2024/tables/schatten.tex
\caption{%
    \topic[0mm]{Unbiased estimator based on Schatten norm.}
    While Schatten norm approximates the largest singular value better than the Frobenius norm, it does not allow the closed-form derivation of sampling strategies.
    In our experiments with numerical constrained optimization, we have not seen any significant improvement over the methods that are based on the Frobenius norm.
}
\label{tab:schatten}
\vspace{2.2ex}
\begin{center}
\begin{small}
\begin{sc}
\resizebox{\textwidth}{!}{
\begin{tabular}{lccccc}
\toprule
\multirow{2}{*}{Strategy}%
& CIFAR-10%
& \multicolumn{2}{c}{CIFAR-100}\\%
& ResNet%
& ResNet%
& CCT\\%
\midrule%
Unbiased & $80.57 \pm 1.39$ & $35.12 \pm 1.51$ & $45.01 \pm 1.07$\\
Collective & $82.59 \pm 0.27$ & $37.83 \pm 1.24$ &  $49.93 \pm 1.77$\\
Schatten ($p=4$) & $80.72 \pm 0.67$  &  $34.54 \pm 0.62$ & $45.75 \pm 0.44$\\
\bottomrule
\end{tabular}
}
\end{sc}
\end{small}
\end{center}

%% file: neurips_2024/appendix/derivations.tex
\section{Optimal Sampling Distribution}
In the main text we discussed estimators obtained from singular vector decomposition. However, we found that the obtained results are more general and apply to any orthogonal decomposition.
Therefore, in this section we provide proofs in a slightly generalized form.

More specifically, we consider a vector space equipped with an inner product and the corresponding induced norm.
Assume that some element of that space $Y$ may be represented as a sum of mutually orthogonal terms $Y = \sum_{i=1}^N X_i$, and the terms are enumerated in descending order w.r.t. their norms, i.e. $\lambda_i = \left\lVert X_i \right\rVert$, and $\lambda_1 \geq \dots \geq \lambda_N > 0$.
In the particular case of SVD, $Y$ is equal to the matrix $W$, $X_i = \lambda_i u_i v_i^T$, and the Frobenius inner product is considered.

\subsection{Unbiased Estimator}
\label{sec:derivation_unbiased_estimator}
Let  $\hat{Y}$ be an estimator of the element $Y$ of the following type,
\begin{align}
\hat{Y} 
= \hat{Y}\lft(\Isubsample\rgt) 
= \sum_{j=1}^n w_j\lft(\Isubsample\rgt) X_{\Isubsample_j}
= \sum_{i=1}^N z_i w_i\lft(\Isubsample\rgt) X_i,
\end{align}
where $\Isubsample \subset \left\{1,\dots,N\right\}$ and the size of $\Isubsample$ equals $n$,  $z_i$ is an indicator variable $z_i = z_i\lft(\Isubsample\rgt) =  \Iindicator\lft(i \in \Isubsample\rgt),$ and ${w_i} \in \mathbb{R}$ are scalar values, potentially depending on the whole subset $\Isubsample$.
Assuming there exists a distribution $p\lft(\Isubsample\rgt)$ over all subsets of size $n$, the expected value of the estimator can be calculated
\begin{align}
    \Eexpect_{p\lft(\Isubsample\rgt)} \hat{Y}\lft(\Isubsample\rgt)
    = \sum_{\Isubsample} p\lft(\Isubsample\rgt) \hat{Y}\lft(\Isubsample\rgt)
    = \sum_{i=1}^N X_i \sum_{\Isubsample} p\lft(\Isubsample\rgt) z_i w_i\lft(\Isubsample\rgt)
    = \sum_{i=1}^N X_i \Eexpect_{p\lft(\Isubsample\rgt)} \left[ z_i w_i\lft(\Isubsample\rgt) \right].
\end{align}
\subsubsection{Horvitz-Thompson Estimator}
The bias of the estimator above equals  $\left\lVert \Eexpect_{p\lft(\Isubsample\rgt)} \hat{Y}\lft(\Isubsample\rgt) - Y \right\rVert =  \left\lVert \sum_{i=1}^N X_i \Eexpect_{p\lft(\Isubsample\rgt)} \left[ z_i w_i\lft(\Isubsample\rgt) - 1 \right] \right\rVert$,
and if (i) all the $\left\{ X_i \right\}_i$ are mutually orthogonal and (ii) the estimator $\hat{Y}$ is unbiased,  the following constraint is obtained
\begin{align}
    \Eexpect_{p\left(\Isubsample\right)} \left[ z_i w_i\lft(\Isubsample\rgt) \right] = 1 \quad \textrm{for } i=1,\dots,N.
\end{align}
Taking into account that $z_i$ is a binary random variable, we can use that $z_i^2 = z_i$ and derive the following
\begin{align}
    1 
    &= \Eexpect \left[ z_i w_i\lft(\Isubsample\right) \rgt] 
    = \Eexpect \Eexpect \left[ z_i w_i \vert z_i \right] 
    = \Eexpect \left[ z_i \Eexpect \left[ w_i \vert z_i \right] \right]
    \\ &= \Pr\lft(z_i = 1\rgt) \Eexpect \left[ w_i \vert z_i = 1 \right] 
    = \pi_i \Eexpect \left[ w_i \vert z_i = 1 \right],
\end{align}
leading to $\Eexpect \left[ w_i \vert z_i = 1 \right] = \pi_i^{-1}.$

Due to the mutual orthogonality of the $\left\{ X_i \right\}_i$
\begin{align}
    \Eexpect_{p\left(\Isubsample\right)} \left\lVert Y - \hat{Y} \right\rVert^2
    =  \sum_{i=1}^N \lambda_i^2  \Eexpect_{p\left(\Isubsample\right)} \left[z_i w_i - 1 \right]^2.
\end{align}
Now it is possible to state the following optimization problem to find the unbiased estimator with the lowest mean squared error
\begin{equation}
\begin{aligned}
    \min\limits_{p\left(\Isubsample\right), \left\{w_i \right\} } \quad & \sum_{i=1}^N \lambda_i^2  \Eexpect_{p\left(\Isubsample\right)} \left[z_i w_i - 1 \right]^2
    \\ \textrm{s.t.} \quad & \Eexpect_{p\left(\Isubsample\right)} \left[ z_i w_i \right] = 1, \quad i=1,\dots,N,
    \\ & \sum_{i=1}^N \pi_i = n, \quad  0 \le \pi_i \le 1.
\end{aligned}
\end{equation}

The last constraint takes place since the size of the subset $\Isubsample$ is set equal to $n$ which implies $\sum_i z_i = n$, and due to the linear property of expectation $\Eexpect \sum_i z_i = \sum_i\Eexpect z_i = \sum_i \pi_i$.
Furthermore, the quadratic term under the expectation can be simplified thanks to the unbiasedness and lower-bounded with the help of Jensen's inequality
\begin{equation}
\begin{aligned}
    \Eexpect \left[z_i w_i - 1 \right]^2 
    & = \Eexpect \left[z_i w_i^2 \right] - 2  \Eexpect \left[z_i w_i \right] + 1
    \\ & = \Eexpect \left[z_i w_i^2 \right] - 1
    \\ & = \pi_i \Eexpect \left[w_i^2 \vert z_i = 1\right] - 1
    \\ & \ge \pi_i \left( \Eexpect \left[w_i \vert z_i = 1\right] \right)^2 - 1
    \\ & = \pi_i \cdot \left(\pi_i^{-1}\right)^2 - 1
    \\ & = \pi_i^{-1} - 1,
\end{aligned}
\end{equation}
and this lower bound can be achieved by setting $w_i$ to a non-random value $w_i = \pi_i^{-1}$.

\subsubsection{Inclusion Probabilities}
With the optimal values of auxiliary  multipliers known, the optimization problem can be rewritten as 
\begin{equation}
\begin{aligned}
    \min\limits_{\left\{ \pi_i\right\}} \quad & \sum_{i=1}^N \lambda_i^2 \left(\pi_i^{-1} - 1\right),
    \\ \textrm{s.t.} \quad & 0 \le \pi_i \le 1 \quad i=1,\dots,N,
    \\ & \sum_{i=1}^N \pi_i = n.
\end{aligned}
\end{equation}
This is a typical contrained optimization problem which can be solved with the help of Lagrangian function $L\lft(\left\{ \pi_i \right\}, \alpha, \beta, \left\{ \gamma_i \right\}, \left\{ \delta_i \right\} \rgt) 
= \alpha \sum_i \lambda_i^2 \pi_i^{-1} + \beta \left( \sum_i \pi_i - n \right) + \sum_i \gamma_i \left(-\pi_i \right) + \sum_i \delta_i \left(\pi_i - 1\right) $, leading to the following equations for $i=1,\dots,N$
\begin{equation}
\begin{aligned}
    &\frac{\partial L}{\partial \pi_i} = - \frac{\alpha \lambda_i^2}{\pi_i^2} + \beta - \gamma_i + \delta_i = 0,
    \\ &\sum_i \pi_i = n,
    \\ &\gamma_i \pi_i = 0,
    \\ &\delta_i \left(\pi_i - 1\right) = 0,
    \\ &\alpha, \, \gamma_i, \, \delta_i \ge 0,
    \\ &\alpha^2 + \beta^2 + \sum_i \gamma_i^2 + \sum_i \delta_i^2 > 0.
\end{aligned}
\end{equation}

\topic{Case $\alpha = 0$.}
In this case $\forall i \; \gamma_i - \delta_i = \beta = \text{const}$. 

If $\exists k \; \delta_k > 0 \Rightarrow \pi_k = 1 \Rightarrow \gamma_k = 0 \Rightarrow \beta = - \delta_k < 0 \Rightarrow \forall i \; \beta = \gamma_i - \delta_i < 0 \Rightarrow \forall i \; \gamma_i < \delta_i $.
But $\gamma_i$ and $\delta_i$ cannot be both greater than zero simultaneously since this would imply that $\pi_i = 0$ and $\pi_i = 1$ at the same time. 
Therefore, the only way for the strict inequality to hold true is to set $\forall i \; \gamma_i = 0,\; \delta_i > 0 \Rightarrow \forall i \; \pi_i = 1.$ 
The last equality leads to contradiction in case $n < N$, which is our main case of interest.

Therefore, $\forall i \; \delta_i = 0 \Rightarrow 0 \leq \gamma_i = \beta = \text{const}.$ 
If $\beta = 0$, this breaks the constraint of existence of at least one non-zero coefficient. 
Therefore, $\forall i \; \gamma_i = \beta > 0 \Rightarrow \forall i \; \pi_i = 0$, and this leads to contradiction given that $n > 0$.

\topic{Case $\alpha = 1$.}
W.l.o.g. we can assume than $\lambda_1 \geq \lambda_2 \geq \dots \geq \lambda_N > 0$.
From our equations, $\forall i \; \pi_i^2 = \dfrac{\lambda_i^2}{\beta + \delta_i - \gamma_i}$.

If $\exists l \; \gamma_l > 0 \Rightarrow \pi_l = 0 \Rightarrow \lambda_l = 0$, and this contradicts the assumptions above.
Therefore, $\forall i \; \gamma_i = 0$, and $\forall i \; \pi_i^2 = \dfrac{\lambda_i^2}{\beta + \delta_i}$.

Now let us assume that $\exists k \; \delta_k > 0 \Rightarrow \pi_k = 1 \Rightarrow \beta = \lambda_k^2 - \delta_k \Rightarrow \forall j \neq k \; 1 \geq \pi_j^2 = \dfrac{\lambda_j^2}{\lambda_k^2 - \delta_k + \delta_j} \Rightarrow \forall j \neq k \; \lambda_j^2 - \lambda_k^2 \leq \delta_j - \delta_k$. Therefore, if $j < k$, then $\lambda_j > \lambda_k > 0$, and consequently $\delta_j > \delta_k > 0$ which leads to $\pi_j = 1$. Summarizing, if $\pi_k = 1, $ then $\pi_1 = \dots = \pi_k = 1$.

Based on the previous observation, we define 
\begin{equation}
    t = \min\limits_i \left\{ 1 \leq i \leq N: \; \pi_i < 1 \right\} - 1,
\end{equation}
which means that $\pi_1 = \dots \pi_t = 1$, and $\pi_{t+1}, \dots, \pi_N < 1 \Rightarrow \delta_l = 0$ for $l=t+1,\dots,N, $ and $\pi_l^2 = \frac{\lambda_l^2}{\beta}$.
After that, the sum of all the inclusion probabilities equals $\sum_{i=1}^N \pi_i = t + \sum_{l=t+1}^N \frac{\lambda_l}{\sqrt{\beta}} = n \Rightarrow \sqrt{\beta} = \frac{\sum_{l=t+1}^{N} \lambda_l}{n - t} \Rightarrow \pi_l = \left(n - t\right) \frac{\lambda_l}{\sum_{j=t+1}^N \lambda_j}.$
Since all the probabilities $\pi_l$ should be less than 1 for $l > t$, there is a natural constraint $\lambda_{t+1} < \frac{\sum_{j=t+1}^N \lambda_j}{n-t}.$

Thus, the search of the optimal set of inclusion probabilities can be done in time not worse than $O\lft(Nn\rgt)$. 
See~\cref{alg:unbiased} for details. 
In practice, we found that the computation speed can be significantly increased with vectorization.
Note that the solution always exists in~\cref{alg:unbiased} because for $t=n-1$ the condition $\lambda_{t+1} < \frac{\Lambda_t}{n-t}$ holds true as obviously $\lambda_n < \sum_{l=n}^N \lambda_l$.
\begin{remark}
The optimal set of probabilities $\left(\pi_1,\dots,\pi_N\right)$ specified by~\cref{eq:unbiased_pi_value} is \emph{balanced}~\citep{tille_sampling_2006} w.r.t. the variables $\left(\lambda_1,\dots,\lambda_N\right)$, i.e. for any subset of indices $\Isubsample$ sampled according to the optimal distribution $p\left(\Isubsample\right)$ the following equality holds
\begin{equation*}
    \sum\limits_{j=1}^n \dfrac{\lambda_{\Isubsample_j}}{\pi_{\Isubsample_j}} 
    = \sum\limits_{i=1}^N \lambda_i.
\end{equation*}
Therefore, in the case of spectral model sharding, before the local training starts, the sum of singular values of the estimator $\hat{W}$ is equal to the sum of singular values of the full matrix $W$.
\end{remark}

\subsection{Collective Estimator}
\label{sec:derivation_collective_estimator}
Now assume a slightly modified version of the setting described above:
There exist $C > 1$ clients each of which samples an i.i.d. estimator $\hat{Y}^{\left(c\right)}, \; 1 \leq c \leq C,$ of the target value $Y$ using the same distribution $p\left(\Isubsample\right)$.
In this section we consider a simpler case of weighting where coefficients $w_i$ do not depend on the sampled subset of indices anymore, i.e. $\hat{Y}^{\left(c\right)} = \sum_{i=1}^N z_i^{\left(c\right)} w_i X_i$.
Imagine another participant, a `server', who tries to reconstruct the target value $Y$ by averaging the clients' estimators, $\bar{Y} = \frac{1}{C} \sum_{c=1}^C \hat{Y}^{\left(c\right)} = \frac{1}{C} \sum_{i=1}^N w_i X_i \sum_{c=1}^C z_i^{\left(c\right)}$.
We aim to minimize the average squared error between the server's estimator $\bar{Y}$ and ground-true value $Y$.
As before, due to the mutual orthogonality of $\left\{X_i\right\}$, this error can be expressed solely in terms of the magnitudes,
\begin{equation}
\begin{aligned}
    \Eexpect \left\lVert Y - \bar{Y} \right\rVert^2 &= \Eexpect \sum_{i=1}^N \left( \lambda_i - \lambda_i \frac{w_i}{C} \sum_{c=1}^C z_i^{\left(c\right)} \right)^2
    \\ &= \sum_{i=1}^N \lambda_i^2 \Eexpect \left(1 - \frac{w_i}{C} \sum_{c=1}^C z_i^{\left(c\right)} \right)^2
    \\ &= \sum_{i=1}^N \lambda_i^2 \left( 1 - \frac{2 w_i}{C} \sum_{c=1}^C \Eexpect z_i^{\left(c\right)} + \frac{w_i^2}{C^2} \left[\sum_{c=1}^C \Eexpect z_i^{\left(c\right)} + 2 \sum_{c' < c} \Eexpect \left\{ z_i^{\left(c\right)} z_i^{\left(c'\right)} \right\} \right] \right)
    \\ &= \sum_{i=1}^N \lambda_i^2 \left( 1 - \frac{2 w_i}{C} C \pi_i + \frac{w_i^2}{C^2} \left[ C\pi_i + 2 \frac{C \left(C-1\right)}{2} \pi_i^2 \right] \right)
    \\ &= \sum_{i=1}^N \lambda_i^2 \left( 1 - 2 w_i \pi_i + \frac{w_i^2}{C} \pi_i + \frac{w_i^2 \left(C-1\right)}{C} \pi_i^2 \right)
    \\ &= \sum_{i=1}^N \lambda_i^2 + \sum_{i=1}^N \lambda_i^2 w_i \pi_i \left(-2 + \frac{w_i}{C} + \frac{w_i \pi_i \left(C-1\right)}{C} \right).
\end{aligned}
\end{equation}
Now we can formulate a new constrained optimization problem,
\begin{equation}
\begin{aligned}
    \min\limits_{\left\{w_i, \pi_i \right\}_i } \quad & \sum_{i=1}^N \lambda_i^2 w_i \pi_i  \left(-2 + \frac{w_i}{C} + \frac{w_i \pi_i \left(C-1\right)}{C} \right),
    \\ \textrm{s.t.} \quad & \sum_{i=1}^N \pi_i = n,
    \\   & 0 \le \pi_i \le 1, \; i=1,\dots,N.
\end{aligned}
\end{equation}
The corresponding Lagrangian function equals 
\begin{align}
\begin{split}
    L\lft( \left\{\pi_i\right\}, \left\{w_i\right\}, \alpha, \beta, \left\{\gamma_i\right\}, \left\{\delta_i\right\} \rgt) &= \alpha \sum_i \lambda_i^2 w_i \pi_i  \left(-2 + \frac{w_i}{C} + \frac{w_i \pi_i \left(C-1\right)}{C} \right)
    \\ &+ \beta \left(\sum_i \pi_i - n \right) 
    \\ &+ \sum_i \gamma_i \left(-\pi_i\right) + \sum_i \delta_i \left(\pi_i - 1\right)
\end{split}
\end{align}

and leads to the following conditions for all $i = 1,\dots,N$
\begin{equation}
\begin{aligned}
    & \frac{\partial L}{\partial \pi_i} = \alpha \lambda_i^2 w_i \left(-2 + \frac{w_i}{C} + \frac{2 w_i \pi_i \left(C-1\right)}{C} \right) + \beta - \gamma_i + \delta_i = 0,
    \\ & \frac{\partial L}{\partial w_i} = \alpha \lambda_i^2 \pi_i \left( -2 + \frac{2 w_i}{C} + \frac{2 w_i \pi_i \left(C-1\right)}{C} \right) = 0,
    \\ &\sum_i \pi_i = n,
    \\ &\gamma_i \pi_i = 0,
    \\ &\delta_i \left(\pi_i - 1\right) = 0,
    \\ &\alpha, \, \gamma_i, \, \delta_i \ge 0,
    \\ &\alpha^2 + \beta^2 + \sum_i \gamma_i^2 + \sum_i \delta_i^2 > 0.
\end{aligned}
\end{equation}
\paragraph{Case $\alpha=0$.} 
This case does not differ from the one of the setting of the unbiased estimator.
\paragraph{Case $\alpha=1$.}
Since $\sum_i \pi_i = n$, there exists such $k$ that the corresponding inclusion probability $\pi_k$ is greater than zero, $\exists k \; \pi_k > 0 \Rightarrow -2 + \frac{2 w_k}{C} + \frac{2 w_k \pi_k \left(C-1\right)}{C} = 0 \Rightarrow w_k = \frac{C}{1 + \pi_k \left(C-1\right)}$.
As we look for $0 < \pi_k \leq 1$, this implies $ 1 \leq w_k < C$.
Also, since $\pi_k > 0$, then $\gamma_k = 0 \Rightarrow \lambda_k^2 w_k \left(-2 + \frac{w_k}{C} + \frac{2 w_k \pi_k \left(C-1\right)}{C} \right) + \beta  + \delta_k = 0 \Rightarrow \lambda_k^2 w_k \left( - \frac{w_k}{C} \right) + \beta  + \delta_k = 0 \Rightarrow \beta = -\delta_k + \frac{\lambda_k^2 w_k^2}{C}$.

If $\pi_k < 1$, then $\delta_k = 0 \Rightarrow w_k = \frac{\sqrt{C \beta}}{\lambda_k} \Rightarrow \pi_k = \frac{1}{C-1} \left[\lambda_k \sqrt{\frac{C}{\beta}} -1 \right]$.
Furthermore, as in this case $1 < w_k < C$, then $\frac{\lambda_k^2}{C} < \beta < \lambda_k^2 C \Rightarrow \max\limits_{k: 0 < \pi_k < 1} \frac{\lambda_k^2}{C} < \beta < \min\limits_{k: 0 < \pi_k < 1} \lambda_k^2 C.$

Otherwise, if $\pi_k = 1$, then $w_k = 1$, and $\beta = -\delta_k + \frac{\lambda_k^2}{C}$, and since $\delta_k \geq 0$, the following inequality holds true, $\beta \leq \min\limits_{k:\, \pi_k = 1} \frac{\lambda_k^2}{C}.$

Combining these observations together, we obtain $\max\limits_{k: 0 < \pi_k < 1} \frac{\lambda_k^2}{C} < \beta \leq \min\limits_{k:\, \pi_k = 1} \frac{\lambda_k^2}{C}$.
This inequality can hold true only if, similarly to the case of unbiased estimator, several largest magnitudes  correspond to inclusion probabilities equal to 1, and all the rest inclusion probabilities are strictly less than 1.

Revisiting the optimization criterion $\sum_{i=1}^N F_i = \sum_{i=1}^N \lambda_i^2 w_i \pi_i  \left(-2 + \frac{w_i}{C} + \frac{w_i \pi_i \left(C-1\right)}{C} \right)$, we can now consider each separate term of the sum.
Starting with the product of the inclusion probability and the weighting coefficient
\begin{equation}
    \pi_i w_i = \begin{cases}
        1  & \textrm{if} \; \pi_i = 1,\\
        \frac{1}{C-1} \left[\lambda_i \sqrt{\frac{C}{\beta}} -1 \right] \cdot \frac{\sqrt{C \beta}}{\lambda_i} = \frac{C}{C-1} - \frac{\sqrt{C\beta}}{\lambda_i \left(C-1\right)}  & \textrm{if} \; 0 < \pi_i < 1, \\
        0 & \textrm{if} \; \pi_i = 0,
    \end{cases}
\end{equation}
we can write down the equation of each term
\begin{equation}
    F_i = \begin{cases}
        - \lambda_i^2
        & \textrm{if} \; \pi_i = 1,\\
        \lambda_i^2 \left(\frac{C}{C-1} - \frac{\sqrt{C\beta}}{\lambda_i \left(C-1\right)} \right) \left(-2 + \frac{\sqrt{\beta}}{\lambda_i \sqrt{C}} + 1 - \frac{\sqrt{\beta}}{\lambda_i \sqrt{C}} \right)  = - \frac{C}{C-1} \lambda_i  \left(\lambda_i - \sqrt{\frac{\beta}{C}} \right)
        & \textrm{if} \; 0 < \pi_i < 1, \\
        0  & \textrm{if} \; \pi_i = 0.
    \end{cases}
\end{equation}

As was derived above, in case $0 < \pi_i < 1$ we have $\beta < \lambda_i^2 C \Rightarrow \lambda_i > \sqrt{\frac{\beta}{C}}$. 
Therefore, $F_i$ is a monotonically decreasing function of $\lambda_i$ with negative values if $0 < \pi_i < 1$.
Since the goal is to minimize $\sum_i F_i$, this implies that it is beneficial to assign non-zero probabilities to greater values of $\lambda_i$.
Therefore, the optimal solution $\left(\pi_1,\dots,\pi_N\right)$ has the following form for some $t \geq 0, \, u \geq 0,$
\begin{itemize}
    \item $\pi_1 = \dots = \pi_t = 1$,
    \item $1 > \pi_{t+1} \geq \dots \geq \pi_{t+u} > 0$,
    \item $\pi_{t+u+1} = \dots = \pi_N = 0$.
\end{itemize}

Since the sum of all inclusion probabilities should be equal to $n$,
\begin{equation*}
    \sum_i \pi_i = t + \frac{\sqrt{C}}{\sqrt{\beta}\left(C-1\right)} \sum_{i=t+1}^{t+u} \lambda_i - \frac{u}{C-1} = n
    \Rightarrow \sqrt{\beta} = \frac{\sqrt{C}\sum_{i=t+1}^{t+u} \lambda_i}{\left(n-t\right)\left(C-1\right) + u}
\end{equation*}
The search for the best solution takes no longer then $O\lft(N^2 n\rgt)$ time, see~\cref{alg:collective} for details.
As before, the calculations may be done in the vectorized form.

\begin{remark}
\label{remark:limit_of_collective_estimator}
Interestingly, for large values of $C$, the optimal values given by~\cref{eq:collective_omega_pi_value} are approaching \emph{some} unbiased Horvitz-Thompson estimator of $W$, in accordance to the probability theory's law of large numbers. 
However, this estimator is not guaranteed to have the properties that are provided by~\cref{th:unbiased_estimator}.
Therefore, for large values of $C$ it is more beneficial to use the Unbiased estimator instead of the Collective one. 
\end{remark}

\begin{remark}
\label{remark:unbiased_collective_estimator}
It is straightforward to show that if the constraint of unbiasedness is applied to the collective estimator, i.e. $\mathbb{E} \bar{W} = W$, then it necessarily implies that $w_i = \pi_i^{-1}$, just as in the case of the unbiased estimator. 
Therefore, the unbiased collective estimator is the same as the unbiased estimator. 
Since \cref{th:collective_estimator} provides  the estimator with the least Frobenius discrepancy among all the collective estimators, the achieved  least discrepancy cannot be greater than the least possible discrepancy among unbiased collective estimators. 
This concludes the proof that the discrepancy of the collective estimator is not greater than the discrepancy of the unbiased estimator. 
However, as \cref{remark:limit_of_collective_estimator} implies, in the limit case $C \to \infty$ these discrepancies become equal.
\end{remark}

%% file: neurips_2024/appendix/wallenius.tex
\section{Wallenius' Distribution and \texttt{numpy.random.choice}}
\label{sec:wallenius}
The official documentation\footnote{\url{https://github.com/numpy/numpy/blob/0a4b2b83eaf3479f352a7fe5a4b378a169ab48f0/numpy/random/mtrand.pyx\#L856-L954}} of NumPy~\citep{harris2020array} currently does not contain a description of how exactly sampling \emph{without} replacement with unequal probabilities is done if the function \texttt{numpy.random.choice} is employed.
Nevertheless, the source code\footnote{\url{https://github.com/numpy/numpy/blob/0a4b2b83eaf3479f352a7fe5a4b378a169ab48f0/numpy/random/mtrand.pyx\#L1026-L1045}} suggests that sampling is conducted in the following way:
\begin{enumerate}
    \item The remaining number of items is sampled \emph{with} replacement from the multinomial distribution with probabilities proportional to the current chances (in the beginning, those chances are equal to the probabilities provided by the user).
    \item The unique items from the sampled subset are added to the output of the function, and the chances of those items are set equal to zero.
    \item The steps above are repeated with the updated chances until the requested number of items is selected.
\end{enumerate}

The analysis of the statistical properties of the output sample seems infeasible if the probabilities given by the user are far from being uniform.
However, the case of non-uniform chances is exactly what we are interested in when working with singular values of weight matrices.

Interestingly, this algorithm from NumPy resembles the description of multivariate Wallenius’ noncentral hypergeometric distribution~\citep{Wallenius_1963,2e0c8582-b4c5-3cc1-897d-1ca8cebe1269,fog_calculation_2008,fog_sampling_2008} which is often defined through the urn problem.
In detail, if one draws items one by one, instead of trying to get the maximum possible number at each step, this would provide a sample from a special case of Wallenius' law.

Clearly, the two procedures are different significantly enough to produce different distributions.
Despite this, we found that in our simulations the mean vector of the Wallenius' distribution was quite a good approximation for  the empirical inclusion probabilities of the output of NumPy function.

%% file: neurips_2024/appendix/licenses.tex
\section{Licenses for Assets}
\begin{enumerate}
    \item CIFAR-10~\citep{Krizhevsky2009LearningML}: unknown
    \item CIFAR-100~\citep{Krizhevsky2009LearningML}: unknown
    \item TinyImagenet~\citep{tinyimagenet}: unknown
    \item Shakespeare~\citep{pmlr-v54-mcmahan17a}: Apache License, Version 2.0
\end{enumerate}

%% file: neurips_2024/appendix/impact_statement.tex
\section*{Impact Statement}
This paper presents work whose goal is to advance the field of federated learning. 
The potential societal consequences include the increased abilities of training more powerful models on a federation of user devices in privacy-friendly manner.
At the moment, we do not see any direct negative consequences of the conducted research.

%% file: neurips_2024/appendix/algorithms_listing.tex
\clearpage\newpage
\begin{minipage}[t]{0.49\textwidth}
    \input{neurips_2024/algorithms/horvitz-thompson}
\end{minipage}%
\hfill%
\begin{minipage}[t]{0.49\textwidth}
    \input{neurips_2024/algorithms/collective}
\end{minipage}%
\vspace{-3ex}

%% file: neurips_2024/algorithms/horvitz-thompson.tex
\begin{algorithm}[H]
\renewcommand{\algorithmiccomment}[1]{\hfill{\# #1}}
\caption{Inclusion Probabilities for the Unbiased Strategy}\label{alg:unbiased}
\begin{algorithmic}[1]
\REQUIRE{
    Magnitudes $\lambda_1  \geq \lambda_2 \geq \dots \geq \lambda_N > 0$, 
    number of samples $0 < n < N$.
}
\ENSURE{Optimal inclusion probabilities $\left(\pi_1^*, \dots, \pi_N^*\right)$}
\STATE $E^* \leftarrow \infty$\COMMENT{Best criterion value}
\FOR{$t = 0$ to $n-1$}
    \STATE $\Lambda_t \leftarrow \sum_{l=t+1}^N \lambda_l$
    \IF{$\lambda_{t+1} < \frac{\Lambda_t}{n-t}$}
        \FOR{$l=1$ to $t$} 
            \STATE $\pi_l \leftarrow 1$
        \ENDFOR
        \FOR{$l=t+1$ to $N$} 
            \STATE $\pi_l \leftarrow \frac{\left(n - t\right) \lambda_l}{\Lambda_t}$
        \ENDFOR
        \STATE $E = \sum_{i=1}^N \lambda_i^2 \left(\pi_i^{-1} - 1\right)$
        \IF{$E < E^*$} 
            \STATE $E^* \leftarrow E$
            \STATE $\left(\pi_1^*, \dots, \pi_N^*\right) \leftarrow \left(\pi_1, \dots, \pi_N\right)$
        \ENDIF
    \ENDIF
\ENDFOR
\end{algorithmic}
\end{algorithm}

%% file: neurips_2024/algorithms/collective.tex
\begin{algorithm}[H]
\renewcommand{\algorithmiccomment}[1]{\\{\# #1}}
\caption{Inclusion Probabilities and Auxiliary Multipliers for the Collective Strategy}\label{alg:collective}
\begin{algorithmic}[1]
\REQUIRE{
    Magnitudes $\lambda_1  \geq \lambda_2 \geq \dots \geq \lambda_N > 0$, 
    number of samples $0 < n < N$, number of clients $C$.
}
\ENSURE{
    Optimal inclusion probabilities $\left(\pi_1^*, \dots, \pi_N^*\right),$
    optimal auxiliary multipliers $\left(w_1^*, \dots, w_N^*\right).$
}
\COMMENT{Consider top-$n$ case first, i.e. $t=n,\; u=0$}
\FOR{$i=1,\dots,n$}
    \STATE $\pi_i \leftarrow 1$
    \STATE $w_i \leftarrow 1$
\ENDFOR
\FOR{$i=n+1,\dots,N$}
    \STATE $\pi_i \leftarrow 0$
    \STATE $w_i \leftarrow 0$
\ENDFOR
\STATE $E^* \leftarrow - \sum_{i=1}^n \lambda_i^2$
\STATE $\left(\pi_1^*,\dots,\pi_N^*\right) \leftarrow \left(\pi_1,\dots,\pi_N\right)$
\STATE $\left(w_1^*,\dots,w_N^*\right) \leftarrow \left(w_1,\dots,w_N\right)$
\COMMENT{Other cases}
\FOR{$t = 0,\dots,n-1$} 
    \FOR{$i=1,\dots,t$} 
        \STATE $\pi_i \leftarrow 1$
        \STATE $w_i \leftarrow 1$
    \ENDFOR
    \STATE $E_t \leftarrow - \sum_{i=1}^t \lambda_i^2$
    \FOR{$u=1,\dots,N-t$} 
        \STATE $\sqrt{\beta} \leftarrow \frac{\sqrt{C}\sum_{i=t+1}^{t+u} \lambda_i}{\left(n-t\right)\left(C-1\right) + u} $
        \IF{$\frac{\lambda_{t+1}}{\sqrt{C}} < \sqrt{\beta} \leq \frac{\lambda_t}{\sqrt{C}} \And \sqrt{\beta} < \lambda_{t+u}\sqrt{C} $} 
            \FOR{$i=t+1,\dots,t+u$} 
                \STATE $\pi_i \leftarrow \frac{1}{C-1} \left[\frac{\lambda_i \sqrt{C}}{\sqrt{\beta}}  -1 \right]$
                \STATE $w_i \leftarrow \frac{\sqrt{C}\sqrt{\beta}}{\lambda_i}$
            \ENDFOR
            \FOR{$i=t+u+1,\dots,N$} 
                \STATE $\pi_i \leftarrow 0$
                \STATE $w_i \leftarrow 0$
            \ENDFOR
            \STATE $E \leftarrow E_t - \frac{C}{C-1}\sum_{i=t+1}^{t+u} \lambda_i \left(\lambda_i - \frac{\sqrt{\beta}}{\sqrt{C}} \right)$
            \IF{$E < E^*$} 
                \STATE $E^* \leftarrow E$
                \STATE $\left(\pi_1^*, \dots, \pi_N^*\right) \leftarrow \left(\pi_1, \dots, \pi_N\right)$
                \STATE $\left(w_1^*, \dots, w_N^*\right) \leftarrow \left(w_1, \dots, w_N\right)$
            \ENDIF
        \ENDIF
    \ENDFOR
\ENDFOR
\end{algorithmic}
\end{algorithm}